\documentclass{article}
\PassOptionsToPackage{numbers, compress}{natbib}

\usepackage{arxiv}
\usepackage[verbose=true,letterpaper]{geometry}
\AtBeginDocument{
  \newgeometry{
    textheight=9in,
    textwidth=7in,
    top=1in,
    headheight=12pt,
    headsep=25pt,
    footskip=30pt
  }
}
\usepackage{times}

\usepackage{nicefrac}       
\usepackage{microtype}      

\usepackage[utf8]{inputenc}
\usepackage[T1]{fontenc}
\usepackage{microtype}
\usepackage{graphicx}
\usepackage{subfigure}
\usepackage{booktabs}
\usepackage[colorlinks=true,citecolor=blue,linkcolor=blue]{hyperref}

\usepackage{url}
\usepackage{nicefrac}
\usepackage{listings}
\usepackage{mathtools}
\usepackage{amsfonts}
\usepackage{amssymb}
\usepackage{amsmath}
\usepackage{amsthm}
\usepackage{docmute}
\usepackage[acronym, nomain]{glossaries}
\usepackage{indentfirst}
\usepackage{bm}
\usepackage{bbm}
\usepackage{here}
\usepackage{booktabs}       
\usepackage{cancel}
\usepackage{multirow}
\usepackage{wrapfig}
\usepackage{lipsum}

\usepackage{natbib}
\usepackage{algorithm,algorithmic}

\newtheorem{theorem}{Theorem}

\newtheorem{corollary}{Corollary}

\newtheorem{definition}{Definition}

\DeclareMathOperator*{\argmin}{\mathop{\rm argmin}}

\DeclarePairedDelimiterX{\KL}[2]{\mathrm{KL}[}{]}{#1\;\delimsize\|\;#2}

\DeclarePairedDelimiterX\braket[2]{\langle}{\rangle}{#1 \delimsize\vert #2}

\newcommand{\kce}{\mathrm{kCE}}

\newcommand{\calx}{\mathcal{X}}
\newcommand{\calg}{\mathcal{G}}

\newcommand{\lsq}{\ell_{\mathrm{sq}}}
\newcommand{\len}{\ell_{\mathrm{xent}}}
\newcommand{\real}{\mathbb{R}}
\newcommand{\lpsi}{\ell^{\psi}}

\newcommand{\pgap}{\mathrm{pGap}}
\newcommand{\lip}{\mathrm{Lip}}

\newcommand{\dCE}{\mathrm{dCE}_\mathcal{D}(f)}

\newcommand{\nte}{n_\mathrm{te}}
\newcommand{\nre}{n_\mathrm{re}}
\newcommand{\ntr}{n_\mathrm{tr}}

\newcommand{\ece}{\mathrm{ECE}}

\newcommand{\smce}{\mathrm{smCE}}
\newcommand{\binece}{\mathrm{binnedECE}}
\newcommand{\Ex}{\mathbb{E}}
\newcommand{\calo}{\mathcal{O}}
\newcommand{\calh}{\mathcal{H}}

\newcommand{\cali}{\mathcal{I}}
\newcommand{\calk}{\mathcal{K}}

\newcommand{\cald}{\mathcal{D}}

\newcommand{\calf}{\mathcal{F}}

\newcommand{\calcald}{\mathrm{cal}(\mathcal{D})}

\usepackage{multirow}

\usepackage[utf8]{inputenc} 
\usepackage[T1]{fontenc}    
\usepackage{url}            
\usepackage{booktabs}       
\usepackage{amsfonts}       
\usepackage{nicefrac}       
\usepackage{microtype}      
\usepackage{lipsum}
\usepackage{fancyhdr}       
\usepackage{graphicx}       
\graphicspath{{media/}}     
\usepackage{mathtools}
\mathtoolsset{showonlyrefs}
\usepackage[capitalize,noabbrev]{cleveref}

\title{$L_2$-Regularized Empirical Risk Minimization Guarantees \\ Small Smooth Calibration Error}
\author{%
  Masahiro Fujisawa$^{*}$\\
  The University of Osaka / RIKEN AIP / Lattice Lab, Toyota Motor Corporation \\
  \texttt{fujisawa@ist.osaka-u.ac.jp} \\
  \And
  Futoshi Futami\thanks{Equal contribution.}\\
    The University of Osaka / RIKEN AIP / The University of Tokyo \\
  \texttt{futami.futoshi.es@osaka-u.ac.jp}
  }

\date{}

\begin{document}
\maketitle

\begin{abstract}
Calibration of predicted probabilities is critical for reliable machine learning, yet it is poorly understood how standard training procedures yield well-calibrated models.
This work provides the first theoretical proof that canonical $L_{2}$-regularized empirical risk minimization directly controls the smooth calibration error (smCE) without post-hoc correction or specialized calibration-promoting regularizer.
We establish finite-sample generalization bounds for smCE based on optimization error, regularization strength, and the Rademacher complexity.
We then instantiate this theory for models in reproducing kernel Hilbert spaces, deriving concrete guarantees for kernel ridge and logistic regression. 
Our experiments confirm these specific guarantees, demonstrating that $L_{2}$-regularized ERM can provide a well-calibrated model without boosting or post-hoc recalibration.
The source code to reproduce all experiments is available at \url{https://github.com/msfuji0211/erm_calibration}.
\end{abstract}

\section{Introduction}\label{sec_intro}
Calibration—the alignment between predicted probabilities and the true label frequency—is essential for building reliable machine learning models, especially in risk-sensitive applications~\citep{Dawid1982, widmann2021calibration}.  In such settings, calibration metrics are central to assessing model reliability~\citep{Kuleshov2015, gneiting2007probabilistic, kuleshov2022calibrated}, motivating the need for theoretical guarantees on the calibration performance of learned models. To achieve well-calibrated models, a variety of post-hoc recalibration techniques~\citep{guo2017calibration, zadrozny2001obtaining, gupta21b} and regularization-based training strategies~\citep{kuleshov2022calibrated, kumar2019verified, dheur2023large, NEURIPS2022_33d6e648} have been proposed. However, most of these methods are supported mainly by empirical evidence, and the theoretical principles underlying how to design algorithms that achieve both high accuracy and low calibration error remain poorly understood.

This work addresses this problem by providing a theoretical analysis of the relationship between empirical risk minimization (ERM)~\citep{Mohri}—a core principle in machine learning—and calibration error. While ERM with proper losses~\citep{gneiting2007strictly, lakshminarayanan2017simple} is widely used for classification due to its ability to achieve high accuracy and approximate true conditional probabilities, recent studies have shown that it can still yield poorly calibrated models, particularly in deep learning~\citep{guo2017calibration}.  This has motivated growing interest in understanding how loss functions and training algorithms influence calibration performance. A key recent development is the notion of the \emph{post-processing gap}~\citep{blasiok2023when}, which characterizes calibration error as the improvement obtainable through Lipschitz transformations. This perspective has attracted increasing attention~\citep{blasiok2023loss, pmlr-v202-globus-harris23a, hansen2024multicalibration}.

Building on this framework, \citet{blasiok2023when} showed that minimizing a proper loss with a structured regularization can yield well-calibrated models. While their analysis provides valuable theoretical insights, it has two key limitations: (1) the optimization is defined over population risk, which is impractical for implementation; and (2) the regularization is abstract and does not correspond to standard training algorithms used in practice.

To address these limitations, in this work, we provide a theoretical analysis showing that calibration error—specifically, the smooth calibration error (CE)~\citep{blasiok2023unify}—can be effectively controlled via $L_2$-regularized ERM. This establishes a principled foundation for designing calibration-aware algorithms using standard ERM. Our main contributions are as follows:

We consider binary classification under $L_2$-regularized ERM with a hypothesis class closed under Lipschitz transformations. In this setting, we show that the smooth CE computed on the training data--referred to as the \emph{training smooth CE}--can be bounded in terms of the regularization coefficient and the optimization error (Theorem~\ref{thm_erm_smooth}). We further establish a generalization bound that quantifies the gap between the smooth CE (defined over the population) and training smooth CE via the Rademacher complexity~\citep{Mohri} of the hypothesis class (Theorem~\ref{thm_erm_smooth_gen}). These results reveal how regularization, optimization, and function class complexity jointly control smooth CE.

To demonstrate practical applicability, we instantiate our theory for models in reproducing kernel Hilbert spaces (RKHS)~\citep{Mohri}.
Our analysis identifies a key distinction based on the input dimensionality: the RKHS induced by the Laplace kernel in one dimension satisfies our closure assumption. 
This leads to guarantees with a faster convergence rate for kernel ridge and logistic regression (Corollaries~\ref{thm_kernel_r1_square_gen} and~\ref{thm_kernel_r1_logistic_gen}), a setting particularly relevant for the practical task of \emph{recalibration}~\citep{zadrozny2001obtaining}.
In contrast, we prove that this property does not hold in higher dimensions or for other universal kernels like the Gaussian kernel (Theorem~\ref{thm_composite_higher_dim}).

Even in these more general settings where the closure property is absent, we prove that smooth CE can still be controlled for any universal kernel, albeit with a slower convergence rate (Theorem~\ref{theorem_universal__kernel}). 
Furthermore, our bounds provide a formal characterization of the bias-variance trade-off governed by the regularization coefficient, where balancing model complexity and goodness-of-fit is critical for achieving a small smooth CE (Corolaries~\ref{cor_erm_sq} and~\ref{col_klr}).

In conclusion, our work shifts the perspective on calibration from a post-hoc correction to an intrinsic property of well-posed, regularized learning. We demonstrate that smooth CE can be controlled via the $L_2$-regularized ERM, providing a principled foundation for designing models that are inherently both accurate and reliable.

\section{Preliminaries}

We denote random variables by uppercase letters (e.g., $X$) and deterministic quantities by lowercase letters (e.g., $x$).  
The Euclidean inner product and norm are denoted by $\cdot$ and $\|\cdot\|$, respectively.

\subsection{Binary Classification}\label{subsec:binary}
We assume that data points $(X,Y) \in \mathcal{Z} = \mathcal{X} \times \mathcal{Y}$ are drawn from a distribution $\cald$, where $\mathcal{X}$ is the input space and $\mathcal{Y} = \{0,1\}$ is the label space.
For a given input $X=x$, the label $Y$ is drawn from a Bernoulli distribution $Y \sim \mathrm{Ber}(v)$ for some $v \in [0,1]$ with the true conditional probability $f^*(x) = \mathbb{E}[Y \mid x]$.
Let $\calf$ be a hypothesis class consisting of functions $f: \mathcal{X} \to [0,1]$.
A predictor $f \in \mathcal{F}$ is trained to approximate $f^*$ based on samples drawn from $\mathcal{D}$.

To evaluate the performance of $f$, we define a loss function $\ell: [0,1] \times \mathcal{Y} \to \mathbb{R}$, denoted by $\ell(v, y)$.
We often adopt a \emph{proper loss}~\citep{Gneiting2007} as $\ell$, where $\ell$ is called \emph{proper} if, for any $v, v' \in [0,1]$, it satisfies $\mathbb{E}_{Y \sim \mathrm{Ber}(v)}[\ell(v, Y)] \leq \mathbb{E}_{Y \sim \mathrm{Ber}(v)}[\ell(v', Y)]$.
This paper focuses on the two such losses:  
the squared loss $\lsq(v, y) \coloneqq (y - v)^2$ and the cross-entropy loss $\len(v, y) \coloneqq -y\log v - (1 - y)\log(1 - v)$. 
We note that our analysis framework is general and can be extended to other proper losses mentioned by \citet{blasiok2023when} (see Appendix~\ref{app_metrics_proper} for details).

For ease of discussion, we introduce the \emph{Savage representation}~\citep{Gneiting2007} of $\len$. 
Let $\psi: \mathbb{R} \to \mathbb{R}$ be the convex function defined by $\psi(s) \coloneqq \log(1 + e^s)$ for $s \in \mathbb{R}$.  
The proper scoring rule induced by $\psi$ is then given by $\lpsi(s, y) \coloneqq \psi(s) - s y$, known as the logistic loss.  
Identifying the score with the logit function $g(x) \coloneqq \log \frac{f(x)}{1-f(x)}$, the corresponding probability can be recovered via the sigmoid function $f(x) = \sigma(g(x)) = 1/(1+\exp(-g(x)))$. 
his yields the equivalence $\lpsi(g(x), y) = \len(f(x), y)$, a representation that facilitates our subsequent analysis.

\subsection{Calibration Metrics and Loss Functions}\label{sec_prelimi_cal}
Given a distribution $\mathcal{D}$, $f$ is \emph{perfectly calibrated} if $\mathbb{E}[Y \mid f(X)] = f(X)$ almost surely. 
While the common binning-based expected calibration error (ECE)~\citep{guo2017calibration} is widely used, it is known to exhibit numerical instabilities and present theoretical difficulties~\citep{kumar2019verified,blasiok2023unify} (see Section~\ref{sec_discuss} for examples).
We therefore focus on the following \emph{smooth CE}~\citep{blasiok2023unify}, a metric with desirable theoretical properties such as continuity and computational efficiency~\citep{blasiok2023unify,hu2024testing} (see Section~\ref{sec_discuss} for details).
\begin{definition}[smooth CE \citep{blasiok2023unify}]
    Let $\mathrm{Lip}_L$ be the set of all $L$-Lipschitz functions from $[0,1]$ to $[-1,1]$. Then, the smooth CE is defined as
    \begin{align}\label{def_weight_cal}
        \smce(f,\cald) \coloneqq \sup_{h \in \mathrm{Lip}_{L=1}} \Ex\left[h(f(X)) \cdot (Y - f(X))\right].
    \end{align}
\end{definition}
For notational simplicity, we denote the class of 1-Lipschitz functions $\mathrm{Lip}_{L=1}$ by $\mathrm{Lip}_1$. Since binning ECE can be both upper- and lower-bounded by smooth CE, we focus on smooth CE (see Section~\ref{sec_discuss}). 
A key property of smooth CE is its characterization via the \emph{post-processing gap} of $f$, which measures the potential improvement in squared loss through Lipschitz post-processing.
\begin{definition}[Post-processing gap under $\lsq$~\citep{blasiok2023when}]
Given a predictor $f\in\calf$ and a distribution $\mathcal{D}$, the post-processing gap under the squared loss is defined as
\begin{align}
    \pgap(f, \mathcal{D}) \coloneqq \mathbb{E}[\ell_{\mathrm{sq}}(f(X), Y)] - \inf_{h \in \mathrm{Lip}_{L=1}} \mathbb{E}[\ell_{\mathrm{sq}}(f(X)+h(f(X)), Y)].
\end{align}
\end{definition}
\begin{theorem}[Theorem~2.4 in \citet{blasiok2023when}]\label{thm:blasiok_pgap}
For any predictor $f\in\calf$ and distribution $\cald$,
\begin{align}
    \smce(f, \cald)^2 \leq \pgap(f, \cald) \leq 2 \smce(f, \cald).
\end{align}
\end{theorem}
This result shows that a predictor is well-calibrated in terms of in terms of $\smce(f, \cald)$ if it cannot be improved by simple post-processing.

In classification models, we often minimize the cross-entropy loss, focusing on a logit function.
Thus, it is natural to consider the following post-processing gap under $\lpsi$.
\begin{definition}[Dual post-processing gap under $\lpsi$~\citep{blasiok2023when}]
Let $\mathrm{Lip}_{1}(\mathbb{R}, [-4,4])$ denote the class of $1$-Lipschitz functions from $\mathbb{R}$ to $[-4,4]$.  
Then, the dual post-processing gap is defined as
\begin{align}
    \pgap&^{(\psi,1/4)}(g, \mathcal{D}) \coloneqq \mathbb{E}[\lpsi(g(X), Y)] - \inf_{h \in \mathrm{Lip}_{1}(\mathbb{R}, [-4,4])} \mathbb{E}[\lpsi( g(X)+h(g(X)) , Y)].
\end{align}
\end{definition}
\begin{definition}[Dual smooth CE~\citep{blasiok2023when}]
Let $\mathrm{Lip}_{1/4}(\mathbb{R}, [-1,1])$ denote the class of $1/4$-Lipschitz functions from $\mathbb{R}$ to $[-1,1]$.  
Given a logit function $g: \mathcal{X} \to \mathbb{R}$ and prediction $f(x) = \sigma(g(x))$, the dual smooth calibration error is defined as
\begin{align}
    \smce^{(\psi,1/4)}(g, \mathcal{D}) \coloneqq \sup_{h \in \mathrm{Lip}_{1/4}(\mathbb{R}, [-1,1])} \mathbb{E}\left[ h(g(X)) \cdot (Y - f(X)) \right].
\end{align}
\end{definition}
We remark that the choice of the Lipschitz constant $1/4$ and the bounded ranges $[-4,4]$ and $[-1,1]$ follow from the smoothness of the softplus function $\psi$. An analogue of Theorem~\ref{thm:blasiok_pgap} holds for $\pgap^{(\psi,1/4)}(g, \mathcal{D})$ and $\smce^{(\psi,1/4)}(g, \mathcal{D})$, as stated in the following corollary.
\begin{corollary}[Corollary~2.4 in~\citet{blasiok2023when}]\label{cor:blasiok_pgap}
The dual post-processing gap and the dual smooth calibration error satisfy the following bounds:
\begin{align}
2 \smce^{(\psi,1/4)}(g, \mathcal{D})^2 \leq \pgap^{(\psi,1/4)}(g, \mathcal{D}) \leq 4 \smce^{(\psi,1/4)}(g, \mathcal{D}).
\end{align}
\end{corollary}
Finally, we remark that the original smooth CE is upper bounded by its dual counterpart:
\begin{align}\label{eq_upperbound_dual_smooth}
\smce(f, \mathcal{D}) \leq \smce^{(\psi,1/4)}(g, \mathcal{D}).
\end{align}

\citet{blasiok2023when} clarified that a certain regularized risk minimization leads to a small smooth CE:
\begin{theorem}[Claim 5.1 \citep{blasiok2023when}]\label{thm_blasic_optimization}
Assume that for any $f \in \mathcal{F}$ and $h \in \mathrm{Lip}_1$, the composition $f + h \circ f$ belongs to $\mathcal{F}$. Let $\mu: \mathcal{F} \to \mathbb{R}^+$ be a complexity measure satisfying $\mu(f + h \circ f) \leq \mu(f) + 1$ for all $f \in \mathcal{F}$ and $h  \in \mathrm{Lip}_1$. For any $\lambda>0$, the minimizer $f^*$ of the regularized minimization problem
\begin{align}\label{eq_structured_blasiok}
f^*=\argmin_{f\in\calf} \Ex_\cald \lsq(f(X_i), Y_i) + \lambda \mu(f)
\end{align}
satisfies $\pgap(f^*, \mathcal{D}) \leq \lambda$, and thus $\smce(f^*, \mathcal{D}) \leq \sqrt{\lambda}$.
\end{theorem}
This result provides a theoretical basis for optimization-based calibration. \citet{blasiok2023when} suggested that structured regularization may align with standard training methods, such as stochastic gradient descent (SGD). However, their formulation relies on \emph{population risk minimization}, which is impractical in real settings, and uses an abstract complexity measure $\mu$ that is hard to implement. To address this, the following section investigates whether smooth CE can be controlled through standard learning algorithms using training data. We show that $L_2$-regularized ERM leads to calibration-aware learning.

\section{Regularized Empirical Risk Minimization for Calibration}
\label{sec:erm_analysis_ce}
In this section, we first extend the theoretical framework of \citet{blasiok2023when} to the context of ERM (Section~\ref{subsec:erm_smce_bound}).  
We then show that two representative frameworks---kernel ridge regression and kernel logistic regression---satisfy the assumptions required for our derived bounds, and we provide smooth CE bounds tailored to each of these methods (Sections~\ref{sec_kernel_ridge} and \ref{sec_logistic}).

\subsection{Smooth CE Analysis under Regularized ERM}\label{subsec:erm_smce_bound}

\textbf{Regularized ERM on $\lsq$:}
Let $\mathcal{F}$ be a function class equipped with a norm $\|\cdot\|_{\mathcal{F}}$ such that $\|f\|_{\mathcal{F}} \leq 1$ for all $f \in \mathcal{F}$. Given a training set $S \coloneqq \{(X_i, Y_i)\}_{i=1}^n$ of independently and identically distributed (i.i.d.) samples from the data distribution $\cald$, we consider the $L_2$-regularized ERM:
\begin{align}\label{eq:krr_argmin}
L_n(f) \coloneqq \frac{1}{n}\sum_{i=1}^n \lsq(f(X_i), Y_i) + \lambda \|f\|_{\mathcal{F}}^2.
\end{align}
We define the optimization error as $\mathrm{err}_n(f) \coloneqq L_n(f) - L_n(f^*_n)$, where $f^*_n \coloneqq \argmin_{f \in \mathcal{F}} L_n(f)$.

\textbf{Analysis of the smooth CE on $S$:}
Before proceeding with our analysis, we define the empirical versions of the smooth CE~\citep{blasiok2023unify}, given a training dataset $S = \{(X_i, Y_i)\}_{i=1}^n$, as follows:
\begin{align}
    \smce(f, S) \coloneqq \sup_{h \in \mathrm{Lip}_{1}} \frac{1}{n}\sum_{i=1}^n\left[h(f(X_i)) \cdot (Y_i - f(X_i))\right].
\end{align}
Our first contribution shows that a small optimization error $\mathrm{err}_n(f)$ and a small regularization parameter $\lambda$ suffice to guarantee a small smooth CE on the training set $S$.
\begin{theorem}\label{thm_erm_smooth}
Assume that for any $f \in \mathcal{F}$ and $h \in \mathrm{Lip}_1$, the composite function $f + h \circ f$ also belongs to $\mathcal{F}$. Then, for any $f \in \mathcal{F}$, we have $\smce(f, S) \leq \sqrt{\lambda + \mathrm{err}_n(f)}$.
\end{theorem}
\begin{proof}[Proof outline]
We relate the empirical post-processing gap, $\pgap(f, S)$, to the regularized objective $L_n$.
By adding and subtracting the regularization term, we can upper-bound the gap using the difference in the objective function values at $f$ and $f + h\circ f$ as follows:
\begin{align}
\pgap(f, S) 
&\leq L_n(f) - \inf_{h \in \mathrm{Lip}_{1}} L_n(f+h\circ f)+\lambda \\
&\leq\!L_n(f^*_n)\!+\!\mathrm{err}_n(f)\!-\!\!\inf_{h \in \mathrm{Lip}_{1}}\!\!L_n(f\!+\!h\!\circ \!f)\!+\!\lambda.
\end{align}
By applying the definition of the optimizer $f^*_{n}$ and the composition assumption ($f + h \circ f \in \mathcal{F}$), we obtain
\begin{align}\label{composite_optimal}
L_n(f^*_n) \leq \text{$\inf_{h \in \mathrm{Lip}_{1}}$} L_n(f + h \circ f),
\end{align}
which concludes the proof.
\end{proof}
Our result can be viewed as the ERM counterpart of Theorem~\ref{thm_blasic_optimization}, providing a bound on the training smooth CE in terms of the optimization error and the regularization coefficient. As in Theorem~\ref{thm_blasic_optimization}, the \emph{composition assumption} that $f + h \circ f \in \mathcal{F}$ plays a central role in our analysis, leading to Eq.~\eqref{composite_optimal}. In the remainder of this paper, we investigate function classes that satisfy this assumption.

\textbf{Analysis of the smooth CE on the unknown $D$:}
While the left-hand side represents the smooth CE on the \emph{training dataset}, our natural objective is to analyze the \emph{population} smooth CE ($\smce(f, \mathcal{D})$) over the unknown $\mathcal{D}$.
The following theorem shows that $\smce(f, \mathcal{D})$ can be upper bounded in terms of the optimization error, the regularization parameter $\lambda$, and the Rademacher complexity.
\begin{theorem}\label{thm_erm_smooth_gen}
Let the empirical and expected Rademacher complexities of a function class $\mathcal{F}$ be defined by $\hat{\mathfrak{R}}_S(\mathcal{F}) \coloneqq \mathbb{E}_{\sigma} \left[\sup_{f \in \mathcal{F}} \frac{1}{n}\sum_{i=1}^n \sigma_i f(X_i)\right]$, and $\mathfrak{R}_{\mathcal{D},n}(\mathcal{F}) \coloneqq \mathbb{E}_{S \sim \mathcal{D}^n}\left[\hat{\mathfrak{R}}_S(\mathcal{F})\right]$, respectively, where $\sigma_i$ are i.i.d.~Rademacher random variables taking values in $\{\pm 1\}$.
Under the assumptions of Theorem~\ref{thm_erm_smooth}, and assuming that $\mathfrak{R}_{\mathcal{D},n}(\mathcal{F})$ is bounded, the following holds with probability at least $1 - \delta$ over the random draw of training data $S$:
\begin{align}
\smce(f, \mathcal{D}) \leq \sqrt{\lambda + \mathrm{err}_n(f)} + 6 \mathfrak{R}_{\mathcal{D},n}(\mathcal{F}) + \sqrt{\frac{\log\tfrac1\delta}{2n}}.
\end{align}
\end{theorem}
This result indicates that when the function class is sufficiently regularized, and both the optimization error and regularization parameter are small, the smooth CE is guaranteed to be small. While Theorem~\ref{thm_blasic_optimization} characterizes the smooth CE in terms of population risk minimization, Theorem~\ref{thm_erm_smooth_gen} provides a clear connection between the population smooth CE and $L_2$-regularized ERM.

However, the composition assumption required by both the analysis of \citet{blasiok2023when} and our own may be overly restrictive in practice.
To examine this limitation, we focus on RKHS, one of the most widely used function classes in machine learning. In particular, focusing on the \emph{kernel ridge regression} (KRR), we show that, except under certain conditions, the composition assumption does not generally hold in this setting (see Table~\ref{tab:our_results}).
This suggests that the relationship between smooth CE and ERM depends critically on the structural properties of the underlying model class.

\subsection{Controlling Smooth CE under KRR}\label{sec_kernel_ridge}
We now describe kernel ridge regression under the squared loss $\lsq$. Let $k: \mathcal{X} \times \mathcal{X} \to \mathbb{R}$ be a kernel with associated RKHS $\calh$ and inner product $\langle \cdot,\cdot\rangle_\calh$, used to learn the predictor $f$. The feature map of $k$ is $\phi:\calx\to\calh$, so that $k(x,x')=\langle \phi(x),\phi(x')\rangle_\calh$ for any $x,x'\in\calx$. We define the hypothesis class as $\calf = \calh \oplus \mathbb{R} = \{f' + b \mid f' \in \calh, b \in \mathbb{R}\}$. The predictor $f$ is then trained by solving the regularized ERM problem in Eq.~\eqref{eq:krr_argmin}. 
The solution admits the following form by the representer theorem~\citep{Mohri}: $f^*_n(x) = \frac{1}{n} \sum_{i=1}^n \alpha_i k(x, x_i) + b$,
where $\{\alpha_i\}_{i=1}^n \subset \mathbb{R}$ and $b \in \mathbb{R}$ is a bias term.

We first focus on the Laplace kernel as a specific choice~\citep{wainwright_2019}. One motivation for this choice is its widespread use in calibration metrics such as the maximum mean calibration error (MMCE)~\citep{pmlr-v80-kumar18a}.
Moreover, the RKHS induced by the Laplace kernel is known to contain Lipschitz continuous functions~\citep{blasiok2023unify}, making it a natural candidate for function classes that satisfy the composition assumption required by Theorem~\ref{thm_erm_smooth_gen}. 
In addition, the Laplace kernel is \emph{universal}~\citep{wainwright_2019}, meaning it has sufficient expressive power to approximate a broad class of functions. We also note that our result extends to general universal kernels, as formalized in Theorem~\ref{theorem_universal__kernel} at the end of this section.

Therefore, this section aims to theoretically investigate whether the Laplace kernel satisfies the composition assumption in Theorem~\ref{thm_erm_smooth_gen}, and whether fitting KRR with the Laplace kernel leads to a small population smooth CE under the setup described above. Our analysis reveals a significant distinction between the following two cases:
(i) the \emph{univariate input space} $\mathcal{X} = \mathbb{R}$, and
(ii) the \emph{multivariate input space} $\mathcal{X} = \mathbb{R}^d$ with $d > 1$.
We present the results for each case in turn. A summary of this section's findings is provided in Table~\ref{tab:our_results}.

\begin{table*}[t]
\centering
\caption{Summary of our theoretical results in kernel ridge regression.}
\label{tab:our_results}
\scalebox{1.}{
\begin{tabular}{ccccc}
\hline
\textbf{Kernel function} & \textbf{Input space} & \textbf{Ex.\ of application} & \textbf{$f + h \circ f \in \calf$ (in Theorem~\ref{thm_erm_smooth})} & \textbf{Corresp.\ thm.} \\ 
\hline\hline
\multirow{2}{*}{Laplace}  
  & $\mathcal{X}=\mathbb{R}$             & Recalibration         & $\checkmark$ & Corollary~\ref{thm_kernel_r1_square_gen}        \\
  & $\mathcal{X}=\mathbb{R}^{d}\ (d>1)$ & Binary Classification & $-$          & Theorem~\ref{theorem_universal__kernel}         \\ 
\hline
\multirow{2}{*}{Gaussian} 
  & $\mathcal{X}=\mathbb{R}$             & Recalibration         & $-$          & Theorem~\ref{theorem_universal__kernel}         \\
  & $\mathcal{X}=\mathbb{R}^{d}\ (d>1)$ & Binary Classification & $-$          & Theorem~\ref{theorem_universal__kernel}         \\
\end{tabular}
}
\end{table*}

\textbf{When $\mathcal{X} = \mathbb{R}$ (Case (i)):}
Given $x,x'\in\calx$, we define the Laplace kernel  as $k(x,x') \coloneqq \exp(-|x - x'|)$. To simplify the notation, we omit the bandwidth. The following theorem shows that, in the univariate case, the Laplace kernel satisfies the composition assumption required by Theorem~\ref{thm_erm_smooth_gen}.
\begin{theorem}\label{composition_on_one_dim}
Assume $\mathcal{X} = \mathbb{R}$ and let $k$ be the Laplace kernel. Then, for any $f \in \mathcal{F} = \calh \oplus \mathbb{R}$ and $h \in \mathrm{Lip}_{1}$, it holds that $f + h \circ f \in \mathcal{F}$.
\end{theorem}
This result allows us to derive the following corollary, which states that minimizing the objective in Eq.~\eqref{eq:krr_argmin} also ensures a small value of the population smooth CE.
\begin{corollary}\label{thm_kernel_r1_square_gen}
Assume that $\mathcal{X} = \mathbb{R}$ and $k$ is the Laplace kernel. Suppose there exist constants $\Lambda > 0$ and $\alpha > 0$ such that $\sup_{x, x' \in \mathcal{X}} k(x, x') \leq \Lambda$, $\|f'\|_{\calh} \leq \alpha$ for any $f'\in\calh$ and $|b|\leq \alpha \Lambda + 1$. Then, with probability at least $1 - \delta$ over the training dataset $S$, for any $f\in\cal{F}$, we have:
\begin{align}
\smce(f, \mathcal{D}) \leq \sqrt{\lambda+\mathrm{err}_n(f)} + 6\frac{3\alpha\Lambda + 2}{\sqrt{n}} + \sqrt{\frac{\log\tfrac1\delta}{2n}}.
\end{align}
\end{corollary}
This theorem illustrates how the regularization coefficient, optimization error, and the complexity of $\calf$ affect the smooth CE. Although \citet{blasiok2023when} conjectured that those factors might influence the smooth CE, we formally provide the first rigorous relationships among them. 

Corollary~\ref{thm_kernel_r1_square_gen} provides a uniform bound over the entire class $\mathcal{F}=\{f'+b|f'\in\mathcal{H},b\in\mathbb{R}\}$ and the complexity of this class seems to depend on $\alpha$, not $\lambda$. However, the fact is that $L_2$-regularized ERM constrains the learned function $f_n^{*}$ to a much smaller subset of $\mathcal{F}$ as follows:
\begin{corollary}\label{cor_erm_sq}
    Under the same settings as Corollary~\ref{thm_kernel_r1_square_gen}, the ERM solution of KRR lies in $\Omega:=\{(f',b)\in\calh \oplus \mathbb{R}|\|f'\|_\calh\leq\lambda^{-1/2},|b|\leq \Lambda/\lambda^{1/2}+1\}$. Then, with probability at least $1 - \delta$ over the training dataset $S$, for any $f\in\Omega$, we have:
\begin{align}
\smce(f, \mathcal{D}) \leq \sqrt{\lambda+\mathrm{err}_n(f)} + 6\frac{3\Lambda + 2}{\sqrt{n\lambda}} + \sqrt{\frac{\log\tfrac1\delta}{2n}}.
\end{align}
\end{corollary}
This bound now explicitly captures the bias-variance trade-off: The first term ($\sqrt{\lambda}$), which relates to the training smooth CE (bias), increases with $\lambda$. The second term ($\mathcal{O}(1/\sqrt{\lambda n})$), which is the Rademacher complexity term (variance), decreases as $\lambda$ increases. This clearly demonstrates how a larger $\lambda$ reduces the complexity of the solution space, thereby tightening the generalization bound, at the cost of increasing the training smooth CE. In Section~\ref{sec:experiments}, we evaluate this trade-off numerically.

In fact, when using the Gaussian kernel, $f + h \circ f \notin \calf$ holds, which means that the composition assumption in Theorem~\ref{thm_erm_smooth_gen} does not hold. This is because the RKHS induced by the Gaussian kernel does not include the Lipshitz functions \citep{blasiok2023when}, therefore $h \circ f \notin \calf$.

A practical application where $\mathcal{X} = \mathbb{R}$ naturally arises is \emph{recalibration}~\citep{guo2017calibration, zadrozny2001obtaining, fujisawa2025}, which aims to adjust poorly calibrated model predictions. See Section~\ref{sec_logistic} for the details.

\textbf{When $\mathcal{X} = \mathbb{R}^d$ (Case (ii)):}
The following theorem demonstrates that in higher dimensions, the RKHS of the Laplace kernel is not closed under post-processing. 
As a result, Theorem~\ref{thm_erm_smooth_gen} does not apply, rendering a direct application of Theorem~\ref{thm_erm_smooth_gen} infeasible.
\begin{theorem}\label{thm_composite_higher_dim}
Let $\mathcal{X} = \mathbb{R}^d$ for $d > 1$ and $k$ be the Laplace kernel. Then for any $f \in \mathcal{F} = \calf \oplus \mathbb{R}$ and any $h \in \mathrm{Lip}_1$, it holds that $f + h \circ f \notin \mathcal{F}$.
\end{theorem}
Similarly to the univariate case, the Gaussian kernel is not closed under post-processing in the multivariate case either, meaning that $f + h \circ f \notin \mathcal{F}$ still holds.

Nevertheless, by leveraging the approximation property of universal kernels in place of the assumption in Theorem~\ref{thm_erm_smooth_gen}, we can still derive a valid (albeit somewhat looser) upper bound on $\smce(f^*_n, \mathcal{D})$.
\begin{theorem}\label{theorem_universal__kernel}
Let $k$ be a universal kernel with associated RKHS $\calh$. Suppose there exist constant $\Lambda$ such that $\sup_{x, x' \in \mathcal{X}} k(x, x') \leq \Lambda$. Then, with probability at least $1 - \delta$ over the draw of the training dataset, the ERM solution of KRR ($f^*_n$) satisfies
\begin{align}
\smce(f^*_n, \mathcal{D}) \leq \sqrt{\lambda + 4\frac{3\Lambda + 2}{\sqrt{\lambda n}} + \sqrt{\frac{2\log\tfrac{2}{\delta}}{n}}}.
\end{align}
\end{theorem}
This theorem holds only for $f^*_n$, whereas Corollary~\ref{thm_kernel_r1_square_gen} applies to any $f \in \calf$. In terms of the upper bound, this theorem yields a rate of $\calo(1/n^{1/4})$, while Corollaries~\ref{thm_kernel_r1_square_gen} and~\ref{cor_erm_sq} achieve $\calo(1/n^{1/2})$. Owing to the relaxed assumptions, however, it applies to a broader class of input spaces $\mathcal{X}$, covering both $\mathcal{X} = \mathbb{R}$ and $\mathcal{X} = \mathbb{R}^d \ (d>1)$. Moreover, this result applies to universal kernels such as the Gaussian kernel. Finally, similar to Corollary~\ref{cor_erm_sq}, there is a bias–variance trade-off concerning $\lambda$.

\subsection{Dual smooth CE and Regularized ERM under Kernel Logistic Regression}\label{sec_logistic}
In this section, following the analytical framework developed in Sections~\ref{subsec:erm_smce_bound} and~\ref{sec_kernel_ridge}, we theoretically analyze the relationship between \emph{dual smooth CE} and regularized ERM with the \emph{cross-entropy loss}, specifically focusing on the setting of \emph{kernel logistic regression} (KLR).

\textbf{Regularized ERM with $\ell^\psi$:}
Let $\mathcal{G}$ denote a class of logit functions $g: \mathcal{X} \to \mathbb{R}$, endowed with a norm $\|\cdot\|_{\mathcal{G}}$.
Given a training dataset $S = {(X_i, Y_i)}_{i=1}^n$, we define the regularized empirical risk as
\begin{align}
L_n(g) \coloneqq \frac{1}{n} \sum_{i=1}^n \ell^\psi(g(X_i), Y_i) + \lambda \|g\|_{\mathcal{G}}^2.
\end{align}
We define $g^{*}_{n} \coloneqq \argmin_{g \in \mathcal{G}} L_{n}(g)$ and $\mathrm{err}_{n}(g) \coloneqq L_n(g) - L_n(g^{*}_{n})$. We now present a generalization bound for the dual smooth CE, analogous to Theorem~\ref{thm_erm_smooth_gen} in the squared loss setting.
\begin{corollary}\label{thm_erm_dual}
Assume that \(\|g\|_\calg\leq G\) for any \(g\). Moreover assume that for any \(g \in \mathcal{G}\) and \(h \in \mathrm{Lip}_{1/4}(\real\times [-1,1])\), it holds that \(g + h \circ g \in \mathcal{G}\), and the Rademacher complexity \(\mathfrak{R}_{\mathcal{D},n}(\mathcal{G})\) is bounded. Then with probability at least \(1 - \delta\) over a draw of \(S\), for any \(g\in\calg\), we have
\begin{align}
  \smce^{(\psi,1/4)}(g, \cald) \leq \sqrt{\lambda G^{2} + \frac{1}{2}\mathrm{err}_n(g)} + 6 \mathfrak{R}_{\mathcal{D},n}(\mathcal{F}) + \sqrt{\frac{\log\tfrac1\delta}{2n}}.
\end{align}
\end{corollary}
Similar to Theorem~\ref{thm_erm_smooth_gen}, the dual smooth calibration error is upper bounded by the optimization error, the regularization coefficient, and the complexity of the hypothesis class.
The key difference is that the bound explicitly involves the norm of the logit function $g$, denoted by $G$.
Moreover, by Eq.~\eqref{eq_upperbound_dual_smooth}, the dual smooth CE upper bounds the standard smooth CE; hence, this corollary also provides a theoretical guarantee for the smooth CE.

\textbf{Kernel logistic regression (KLR):}
Here, we analyze how KLR controls the dual smooth CE.
Following the setup in Section~\ref{sec_kernel_ridge}, we adopt the same kernel notation and consider the hypothesis class $\mathcal{G} = \mathcal{H} \oplus \mathbb{R} = \{g' + b \mid g' \in \mathcal{H}, b \in \mathbb{R}\}$.
As with the squared loss setting in Section~\ref{sec_kernel_ridge}, we focus on the Laplace kernel and examine two separate cases: $\mathcal{X} = \mathbb{R}$ and $\mathcal{X} = \mathbb{R}^d$.
We begin with the univariate case, $\mathcal{X} = \mathbb{R}$:
\begin{corollary}\label{thm_kernel_r1_logistic_gen}
Assume $\mathcal{X} = \real$ and $k$ is the Laplace kernel. Suppose there exist constants $\Lambda$ and $G$ such that $\displaystyle \sup_{x,x' \in \mathcal{X}} k(x,x') \leq \Lambda$, $\|g'\|_{\mathcal{\calh}} \leq G$ for all $g'\in\calh$ and and $|b|\leq G\Lambda + 1$. Then, for any $g\in\calg$, we have
\begin{align}
\smce^{(\psi,1/4)}(g, \mathcal{D}) \leq \sqrt{\lambda G + \frac{1}{2}\mathrm{err}_n(g)} + 6\frac{3G\Lambda + 2}{\sqrt{n}} + \sqrt{\frac{\log\tfrac1\delta}{2n}}.
\end{align}
\end{corollary}
This result parallels Corollary~\ref{thm_kernel_r1_square_gen} for the squared loss, where the optimal solution $f^{*}_{n}$ is available in closed form. In contrast, under the cross-entropy loss, no such closed-form solution exists for $g^{*}_{n}$. Nevertheless, due to the strong convexity of the objective function $L_n(g)$, the solution $g^{*}_{n}$ can still be computed efficiently, and the optimization error $\mathrm{err}_n(g)$ remains controllable. Assuming such an ERM solution is available, we obtain the following result:
\begin{corollary}\label{col_klr}
    Under the same settings as Corollary~\ref{thm_kernel_r1_logistic_gen}, the ERM solution of KLR lies in $\Omega':=\{(g',b)\in\calh \oplus \mathbb{R}|\|g'\|_\calh\leq \lambda^{-1/2},|b|\leq \Lambda /\lambda^{1/2}+1\}$. Then, with probability at least $1 - \delta$ over the training dataset $S$, for any $g\in\Omega'$, we have:
\begin{align}
\smce^{(\psi,1/4)}(g, \cald) \!\leq\!\sqrt{\lambda\!+\!\mathrm{err}_n(g)} + 6\frac{3\Lambda\! + \!2}{\sqrt{n\lambda}}\! + \!\sqrt{\frac{\log\tfrac1\delta}{2n}}.
\end{align}
\end{corollary}
Similar to Corollary~\ref{cor_erm_sq}, this bound now explicitly captures the bias-variance trade-off concerning $\lambda$.

When $\mathcal{X} = \mathbb{R}^d$, an analogous result to Theorem~\ref{theorem_universal__kernel} holds for any universal kernel.
This means that even without the closure property, the dual smooth CE is controllable, albeit with the slower $\calo(n^{-1/4})$ convergence rate.
The complete statement and proof are presented in Appendix~\ref{app_proos}.

\textbf{Application to recalibration:} 
As discussed in Section~\ref{sec_kernel_ridge}, the case of $\mathcal{X} = \mathbb{R}$ is particularly relevant for recalibration.
In binary classification, recalibration typically involves training a parametric corrector that takes the model's predicted probability for the top label, represented as a scalar, as input and outputs a calibrated probability estimate.
A classical example of this approach is Platt scaling, which fits a logistic regression model to the output scores of a base classifier~\citep{platt1999probabilistic}.

We now outline the recalibration procedure. Let $g$ be a pre-trained logit function that need not belong to an RKHS. If the predictor $f = \sigma(g)$ is miscalibrated, we apply a post-hoc correction using a function $\eta: \real \to [0,1]$, trained on a separate recalibration dataset. Given $g$ and $\mathcal{S}_{\mathrm{re}} = \{(\tilde{X}_i, \tilde{Y}_i)\}_{i=1}^{n_{\mathrm{re}}}$, we form $\mathcal{S}'_{\mathrm{re}} = \{(g(\tilde{X}_i), \tilde{Y}_i)\}_{i=1}^{n_{\mathrm{re}}}$ and train $\eta$ on $\mathcal{S}'_{\mathrm{re}}$. The recalibrated predictor is $\eta \circ g$, with predicted probability $\sigma(\eta(g(x)))$. Since $\{g(\tilde{X}_i)\}$ is one-dimensional and independent, assuming $\eta$ is KLR, Corollary~\ref{thm_kernel_r1_logistic_gen} applies directly. To our knowledge, this gives the first theoretical guarantee for the recalibration performance of KLR. We empirically evaluate kernel-based recalibration in Section~\ref{subsec:exp_recalibration}.

\begin{figure*}[t]
    \centering
    \includegraphics[width=\textwidth]{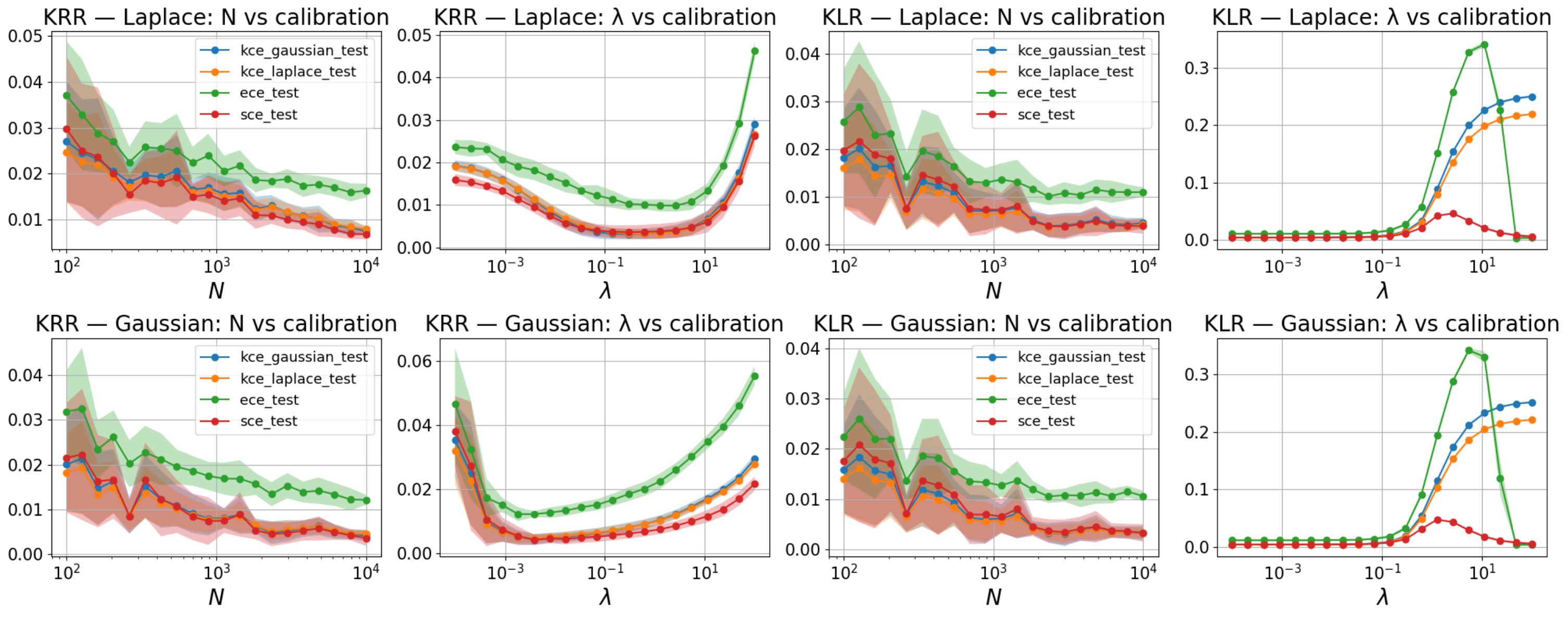}
    \caption{Experimental Results: Effect of Sample Size and Regularization on Calibration Metrics on the toy dataset.}
    \label{fig:res_synthesize}
\end{figure*}

\section{Related Work}\label{sec_discuss}

\textbf{Kernel-based calibration metrics:}
Our work connects closely with kernel-based metrics, particularly the Maximum Mean Calibration Error (MMCE)~\citep{pmlr-v80-kumar18a}. MMCE is defined as $\kce(f, \mathcal{D}, \mathcal{H}_1) \coloneqq \sup_{\phi \in \mathcal{H}_1} \mathbb{E}\left[\phi(f(X)) \cdot (Y - f(X))\right]$, where the supremum is taken over functions in an RKHS, instead of the Lipschitz functions used in smooth CE. The choice of kernel is critical. For the Laplace kernel, the RKHS contains Lipschitz functions, establishing a direct inequality: $\frac{1}{3} \smce(f, \mathcal{D}) \leq \kce(f, \mathcal{D}, k)$. This connection reinforces our theoretical finding in Theorem~\ref{composition_on_one_dim} that the Laplace kernel's RKHS is well-suited for calibration analysis. While \citet{pmlr-v80-kumar18a} proposed algorithms that explicitly regularize with an MMCE term, our work is fundamentally different: we show that standard $L_2$-regularization---a much simpler and more common technique--is sufficient to control smCE.

\textbf{Theoretical guarantees for calibration:}
Our primary contribution is providing theoretical guarantees for calibration under a standard training procedure. The most related work is that of \citet{blasiok2023when}, whose population-level analysis served as the starting point for our work, as detailed in Section~\ref{sec_prelimi_cal}. While they analyze the finite-sample estimation of smCE, we provide the first explicit algorithmic conditions under which the training smCE itself is guaranteed to be small (Theorem~\ref{thm_erm_smooth}). Other works have focused on the generalization properties of the binning-based ECE~\citep{futami2024information}, but did not identify how to modify the training algorithm to ensure a small training calibration error. Our work fills this gap by proving that canonical $L_2$-regularized ERM provides precisely such a guarantee.

\textbf{Recalibration methods:}
Finally, our results provide new theoretical insights into recalibration. This task naturally reduces the input space to a single dimension, making our one-dimensional analysis (e.g., Corollary~\ref{thm_kernel_r1_square_gen}) directly applicable. 
Previous theoretical guarantees for recalibration have largely focused on non-parametric methods~\citep{gupta2020,gupta21b}. 
In contrast, our work provides the first theoretical justification for using parametric models like kernel logistic regression for recalibration, explaining the empirical success of such methods in real-world settings.

\begin{figure*}[t]
    \centering
    \includegraphics[width=\textwidth]{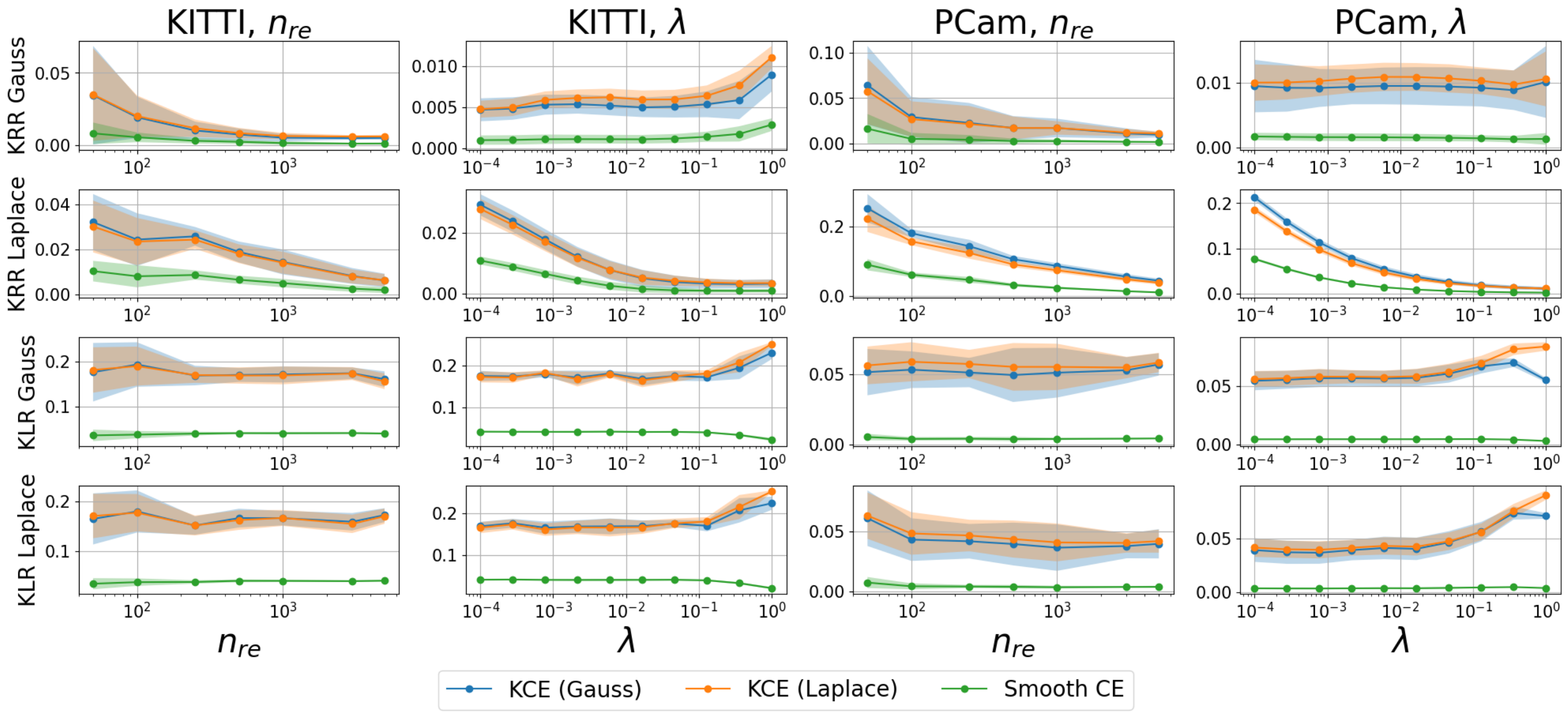}
    \caption{Effect of sample size and regularization on calibration metrics on the KITTI and PCam datasets ($1$-layer neural networks).}
    \label{fig:res_ce_all_nn}
\end{figure*}

\section{Experiment}\label{sec:experiments}
In this section, we conduct experiments to empirically validate our theoretical findings. 
We focus on two key settings that directly correspond to our main theoretical results: (i) a toy dataset where the input space is univariate ($\mathcal{X} = \mathbb{R}$), designed to test our tightest bounds (e.g., Corollaries~\ref{cor_erm_sq} and~\ref{col_klr}), and (ii) a practical recalibration task on real-world data, which also reduces to a one-dimensional problem.
Experiments on higher-dimensional classification tasks ($\mathcal{X} = \mathbb{R}^d$) are presented in Appendix~\ref{app:additional_res}, with a comprehensive summary of all settings and implementation details in Appendix~\ref{app:detail_exp}.

\subsection{Toy Data Experiments (\texorpdfstring{$\calx = \mathbb{R}$}{Lg})}
To numerically validate our theoretical findings, we conduct experiments on a one-dimensional synthetic dataset where the ground-truth conditional probability is analytically tractable. Our theory makes two key predictions that we test here: (i) the existence of a bias-variance trade-off governed by the regularization parameter $\lambda$ (Corollaries~\ref{cor_erm_sq} and~\ref{col_klr}), and (ii) the convergence of smCE as the sample size $n$ increases (Theorem~\ref{theorem_universal__kernel}). 
We designed two experiments to investigate these phenomena for both KRR and KLR. For comparison, we also plot the binning-based ECE and MMCE alongside smCE in Figure~\ref{fig:res_synthesize}.

\textbf{Effect of $\lambda$:}
First, we fix the sample size n and vary $\lambda$. 
For KRR, the results (Figure~\ref{fig:res_synthesize}, second panel) show a distinct U-shaped curve. This is an empirical manifestation of the bias-variance trade-off described in Corollary~\ref{cor_erm_sq}: for small $\lambda$, the model overfits (high variance), while for large $\lambda$, it underfits (high bias). An optimal $\lambda$ balances these two sources of error to achieve the lowest smCE.

For KLR (Figure~\ref{fig:res_synthesize}, rightmost panel), we observe an analogous U-shaped curve. Interestingly, as $\lambda$ becomes extremely large, the predictor collapses to the marginal probability $P(Y)$. While this model is uninformative, it is perfectly calibrated, causing both smCE and ECE to approach zero.

\textbf{Effect of $n$:}
Next, we investigate the effect of the sample size $n$. For KRR, we set $\lambda$ according to the rates that minimize our theoretical bounds: $\lambda = O(n^{1/2})$ for the Laplace kernel and $\lambda = O(n^{1/3})$ for the Gaussian kernel.
As predicted by our theory, Figure~\ref{fig:res_synthesize} (left panel) shows that smCE consistently decreases as $n$ grows. 
For KLR, obtaining a closed-form solution for the optimal $\lambda$ is intractable, so we fix it at a constant value ($\lambda=0.01$). The results (third panel) show that smCE decreases as $n$ increases, which is consistent with the reduction of the generalization error term in our bounds.

\subsection{Recalibration on Real-World Datasets  (\texorpdfstring{$\calx = \mathbb{R}$}{Lg})}\label{subsec:exp_recalibration}
We now validate our one-dimensional theoretical analysis in a practical setting: post-hoc recalibration as discussed in Section~\ref{sec:erm_analysis_ce}.
We trained a one-layer neural network on two benchmark datasets, KITTI~\citep{Geiger12} and PCam~\citep{Veeling18}. We then used the outputs of this base model to form a one-dimensional recalibration dataset, $\{(f(X_i), Y_i)\}_{i=1}^{n_{\mathrm{re}}}$, on which we applied KRR and KLR as recalibrators. We evaluated the effect of both the recalibration sample size $n_{\mathrm{re}}$ and $\lambda$.

The empirical results, shown in Figure~\ref{fig:res_ce_all_nn}, support our theoretical guarantees by confirming two key predictions. 
We observe that as the number of recalibration samples $n_\mathrm{re}$ increases, the smooth CE consistently decreases across all settings. This validates the convergence behavior predicted by our bounds, where the generalization error term diminishes as the sample size grows. Furthermore, the results illustrate the predicted trade-off with respect to the regularization parameter $\lambda$. 
We see that increasing $\lambda$ beyond an optimal point generally leads to a higher smooth CE, which is consistent with our theoretical bias-variance trade-off where excessive regularization increases the bias term. These findings underscore the practical relevance of our theory, demonstrating that our analysis accurately describes the behavior of kernel-based recalibration in real-world applications.

\section{Conclusion and Limitation}\label{sec_conclusion}
This work establishes a foundational link between a standard machine learning principle---$L_2$-regularized ERM---and model calibration. 
We demonstrated that this canonical training procedure directly controls the smCE for binary classification. Our finite-sample bounds, instantiated for kernel methods, formally characterize the interplay between regularization, optimization, and model complexity.
Our analysis, as a first step, highlights important directions for future research. A key challenge is to improve the convergence rates of our bounds for the more general function classes that do not satisfy the strict closure assumption (as in Theorem~\ref{theorem_universal__kernel}). 
Furthermore, extending our framework from the binary to the multiclass setting---leveraging recent developments in multiclass calibration metrics~\citep{gopalan2024computationally}---is a critical next step.

\clearpage
\section*{Acknowledgments}
MF was supported by KAKENHI Grant Number 25K21286, Japan.
FF was supported by JSPS KAKENHI Grant Number JP23K16948, Japan.
FF was supported by JST, PRESTO Grant Number JPMJPR22C8, Japan. 

\appendix
\bibliography{main}
\bibliographystyle{plainnat}

\section{Proofs}\label{app_proos}
\subsection{Proof of Theorem~\ref{thm_erm_smooth}}
First, we define the empirical post-processing gap 
\begin{align}
    \pgap(f, S)&\coloneqq  \frac{1}{n}\sum_{i=1}^n \ell_{\mathrm{sq}}(f(X_i), Y_i) - \inf_{h \in \mathrm{Lip}_{L=1}}  \frac{1}{n}\sum_{i=1}^n \ell_{\mathrm{sq}}(f(X_i)+h(f(X_i)), Y_i)].
\end{align}
We can prove that
\begin{align}
    \smce(f, S)^2 \leq \pgap(f, S) \leq 2 \smce(f, S).
\end{align}
The proof follows directly from Lemma 4.7 in \citet{blasiok2023when}, with the only modification of replacing the population expectation under $\cald$ with the empirical expectation in the last step of their proof.

Accordingly, we upper bound the empirical post-processing gap using the $L_2$-regularized ERM objective evaluated at $f$ and $f + h \circ f$. For notational simplicity, we denote $\ell_{\mathrm{sq}}$ by $\ell$ and write $r \circ f \coloneqq f + h \circ f$.

Then we have
\begin{align}
&\pgap(f, S)\\
&=\frac{1}{n}\sum_{i=1}^n\ell(f(X_i),Y_i)-\inf_{h \in \mathrm{Lip}_{L=1}}\frac{1}{n}\sum_{i=1}^n\ell(r\circ f(X_i),Y_i)\\
&=\frac{1}{n}\sum_{i=1}^n\ell(f(X_i),Y_i)+\lambda\|f\|_\calf^2-\lambda\|f\|_\calf^2-\inf_{h \in \mathrm{Lip}_{L=1}}\frac{1}{n}\sum_{i=1}^n\ell(r\circ f(X_i),Y_i)\\
&=\frac{1}{n}\sum_{i=1}^n\ell(f(X_i),Y_i)+\lambda\|f\|_\calf^2-\lambda\|f\|_\calf^2-\inf_{h \in \mathrm{Lip}_{L=1}}\left(\frac{1}{n}\sum_{i=1}^n\ell(r\circ f(X_i),Y_i)+\lambda\|r\circ f\|_\calf^2-\lambda\|r\circ f\|_\calf^2\right)\\
&=L_n(f)-\inf_{h \in \mathrm{Lip}_{L=1}}\left(\frac{1}{n}\sum_{i=1}^n\ell(r\circ f(X_i),Y_i)+\lambda\|r\circ f\|_\calf^2+\lambda\|f\|_\calf^2-\lambda\|r\circ f\|_\calf^2\right)\\
&\leq L_n(f)-\inf_{h \in \mathrm{Lip}_{L=1}}\left(\frac{1}{n}\sum_{i=1}^n\ell(r\circ f(X_i),Y_i)+\lambda\|r\circ f\|_\calf^2\right)-\inf_{h \in \mathrm{Lip}_{L=1}}\left(\lambda\|f\|_\calf^2-\lambda\|r\circ f\|_\calf^2\right)\\
&\leq L_n(f^*_n)+\mathrm{err}_n(f)-\inf_{h \in \mathrm{Lip}_{L=1}}\left(\frac{1}{n}\sum_{i=1}^n\ell(r\circ f(X_i),Y_i)+\lambda\|r\circ f\|_\calf^2\right)+\lambda\\
&=L_n(f^*_n)+\mathrm{err}_n(f)-\inf_{h \in \mathrm{Lip}_{L=1}}L_n(r\circ f)+\lambda\\
&\leq \mathrm{err}_n(f)+\lambda,
\end{align}
where we used that $L_n(f)=L_n(f^*_n)+\mathrm{err}_n(f)$ and
    $\|r\circ f\|_\calf^2\leq 1$, which is a direct consequence of the assumption on the norm of $\calf$ and $\|f\|_{\calf}\leq 1$. In the final line, we also used the following relation:
\begin{align}
 L_n(f^*_n)\leq \inf_{h \in \mathrm{Lip}_{L=1}}L_n(r\circ f)
\end{align}
since $f^*_n$ minimizes $L_n(f)$ over $\mathcal{F}$ and we assume $f + h \circ f \in \mathcal{F}$.

\subsection{Proof of Theorem~\ref{thm_erm_smooth_gen}}\label{sec_proof_rademacher}
\begin{proof}
We follow the proof technique of Theorem 9.5 in \citet{blasiok2023unify};
\begin{align}
    &\smce(f, \mathcal{D})-\smce(f, S)\\
    &=\sup_\eta \Ex\eta(f(X))\cdot (Y-f(X))-\sup_\eta'\frac{1}{n}\sum_i \eta'(f(X_i))\cdot (Y_i-f(X_i))\\
    &=\sup_\eta\Ex\left[\frac{1}{2}(Y-(f-\eta(f)))^2-\frac{1}{2}(Y-f)^2-\frac{1}{2}\eta(f)^2\right]\\
    &-\sup_{\eta'}\frac{1}{n}\sum_i\left[\frac{1}{2}(Y_i-(f_i-\eta'(f_i)))^2-\frac{1}{2}(Y-f_i)^2-\frac{1}{2}\eta'(f_i)^2\right]\\
    &\leq \sup_\eta\Big( \Ex\left[\frac{1}{2}(Y-(f-\eta(f)))^2-\frac{1}{2}(Y-f)^2-\frac{1}{2}\eta(f)^2\right]\\
    &-\frac{1}{n}\sum_i\left[\frac{1}{2}(Y_i-(f_i-\eta(f_i)))^2-\frac{1}{2}(Y-f_i)^2-\frac{1}{2}\eta(f_i)^2\right]\Big)
\end{align}
and then we take the supremum for $f$,
\begin{align}
    &\sup_{f\in\calf}\smce(f, \mathcal{D})-\smce(f, S)\\
    &\leq \sup_{f\in\calf}\sup_\eta\Big( \Ex\left[\frac{1}{2}(Y-(f-\eta(f)))^2-\frac{1}{2}(Y-f)^2-\frac{1}{2}\eta(f)^2\right]\\
    &-\frac{1}{n}\sum_i\left[\frac{1}{2}(Y_i-(f_i-\eta(f_i)))^2-\frac{1}{2}(Y-f_i)^2-\frac{1}{2}\eta(f_i)^2\right]\Big).
\end{align}
By setting 
\begin{align}
    \Phi(S)=\sup_{f\in\calf}\sup_\eta\Ex&\left[\frac{1}{2}(Y-(f-\eta(f)))^2-\frac{1}{2}(Y-f)^2-\frac{1}{2}\eta(f)^2\right]\\
    &-\frac{1}{n}\sum_i\left[\frac{1}{2}(Y_i-(f_i-\eta(f_i)))^2-\frac{1}{2}(Y-f_i)^2-\frac{1}{2}\eta(f_i)^2\right],
\end{align}
From the proof of Theorem 3.3 in \citet{Mohri}, we have, with probability at least $1 - \delta$,
\begin{align}
\Phi(S) \leq \Ex_S[\phi(S)] + \sqrt{\frac{\log \frac{1}{\delta}}{2n}}.
\end{align}
This can be shown using McDiarmid's inequality, following the same argument as in the proof of Theorem 3.3 in \citet{Mohri} since 
$$
\eta(f(X))\cdot (Y-f(X))=\frac{1}{2}(Y_i-(f_i-\eta(f_i)))^2-\frac{1}{2}(Y-f_i)^2-\frac{1}{2}\eta(f_i)^2
$$
and the constants for the bounded differences in McDiarmid's inequality are equal to $1$, and thus the result is identical to that appearing in Theorem 3.3 in \citet{Mohri}.

We then set $\omega(f,\eta,Y,X)=\frac{1}{2}(Y-(f-\eta(f)))^2-\frac{1}{2}(Y-f)^2-\frac{1}{2}\eta(f)^2$, and by the standard symmetrization argument, we have
\begin{align}
    \Ex_S[\phi(S)]\leq 2\Ex_{\sigma,S}\left[\sup_{f}\sup_\eta \frac{1}{n}\sum_{i=1}^n\sigma_i\omega(f,\eta,Y_i,X_i) \right].
\end{align}
Then by the property of the supremum, 
\begin{align}&\Ex_{\sigma,S}\left[\sup_{f}\sup_\eta \frac{1}{n}\sum_{i=1}^n\sigma_i\omega(f,\eta,Y_i,X_i) \right]\\
    &\leq \frac{1}{2}\Ex_{\sigma,S}\left[\sup_{f}\sup_\eta \frac{1}{n}\sum_{i=1}^n\sigma_i(Y_i-(f_i-\eta(f_i)))^2 \right]+\frac{1}{2}\Ex_{\sigma,S}\left[\sup_{f}\sup_\eta \frac{1}{n}\sum_{i=1}^n\sigma_i(Y_i-f_i)^2 \right]\\
    &+\frac{1}{2}\Ex_{\sigma,S}\left[\sup_{f}\sup_\eta \frac{1}{n}\sum_{i=1}^n\sigma_i\eta(f_i)^2 \right].
\end{align}
Then, from Propositions 11.2 and 11.2 in \citet{Mohri},
\begin{align}
    \frac{1}{2}\Ex_{\sigma,S}\left[\sup_{f}\sup_\eta \frac{1}{n}\sum_{i=1}^n\sigma_i(Y_i-(f_i-\eta(f_i)))^2 \right]\leq \Ex_{\sigma,S}\left[\sup_{f}\sup_\eta \frac{1}{n}\sum_{i=1}^n\sigma_i(f_i-\eta(f_i)) \right]\leq \mathfrak{R}_{\cald,n}(\calf)
\end{align}
and
\begin{align}
    \frac{1}{2}\Ex_{\sigma,S}\left[\sup_{f}\sup_\eta \frac{1}{n}\sum_{i=1}^n\sigma_i(Y_i-f_i)^2 \right]\leq \Ex_{\sigma,S}\left[\sup_{f}\sup_\eta \frac{1}{n}\sum_{i=1}^n\sigma_if_i\right]\leq \mathfrak{R}_{\cald,n}(\calf)
\end{align}
and
\begin{align}
    \frac{1}{2}\Ex_{\sigma,S}\left[\sup_{f}\sup_\eta \frac{1}{n}\sum_{i=1}^n\sigma_i\eta(f_i)^2 \right]\leq \Ex_{\sigma,S}\left[\sup_{f}\sup_\eta \frac{1}{n}\sum_{i=1}^n\sigma_i\eta(f_i)\right]\leq \mathfrak{R}_{\cald,n}(\calf)
\end{align}
holds. In conclusion, we have
\begin{align}
    \Ex_S[\phi(S)]\leq 6\mathfrak{R}_{\cald,n}(\calf)
\end{align}
where we used the composition assumption that $f+\eta\circ f\in\calf$ in the last inequality. This concludes the proof.
\end{proof}

\subsection{Proof of Theorem~\ref{composition_on_one_dim}}

\begin{proof}
From Proposition 7 in \citet{rakhlin19a}, the RKHS associated with the Laplace kernel on $\real$ corresponds to the Sobolev space $H^{1} = W^{2,1}$ (see also \citet{buchholz2022kernel}). Then, by Theorem 6 in \citet{bourdaud2022introduction}, the composition of a Lipschitz function with a function in $H^{1}$ remains in $H^{1}$. To apply Theorem 6, the Lipschitz function $r$ must satisfy $r(0) = 0$. If $r(0) \neq 0$, define $\tilde{r}(x) = r(x) - r(0)$, so that $\tilde{r} \circ f \in H^{1}$. Since constant functions are included in the Sobolev space, we have $r \circ f(x) = \tilde{r} \circ f(x) + r(0) \in H^{1}$.

Similarly, Theorem~7 in \citet{bourdaud2022introduction} shows that the composition of $k \in H^{1}$ and $f \in H^{1}$ yields $k \circ f \in H^{1}$.
\end{proof}

\subsection{Proof of Corollary~\ref{thm_kernel_r1_square_gen} and \ref{cor_erm_sq}}\label{app_proof_kernel}
\begin{proof}

First assume that $|b| \leq \alpha\Lambda + 1$, which will be shown later. Then the corresponding Rademacher complexity is
\begin{align}
\hat{\mathfrak{R}}_S(\mathcal{F}) &= \frac{1}{n} \mathbb{E}_\sigma \left[ \sup_{(f,b) \in \mathcal{F}} \sum_{i=1}^n \sigma_i (f(x_i) + b) \right]\\
&= \frac{1}{n} \mathbb{E}_\sigma \left[ \sup_{\|f\|_{\mathcal{H}} \leq \alpha,\, |b| \leq \alpha\Lambda + 1} \sum_{i=1}^n \sigma_i f(x_i) + \sum_{i=1}^n \sigma_i b \right]\\
&\leq \frac{1}{n} \mathbb{E}_\sigma \left[ \sup_{\|f\|_{\mathcal{H}} \leq \alpha} \sum_{i=1}^n \sigma_i f(x_i) \right]
+ \frac{1}{n} \mathbb{E}_\sigma \left[ \sup_{|b| \leq \alpha\Lambda + 1} \sum_{i=1}^n \sigma_i b \right].\label{rademacher1}
\end{align}

On the basis of the above, the first term of Eq.~\eqref{rademacher1} can be bounded by \( \alpha\Lambda / \sqrt{n} \) from the standard derivation of the Rademacher complexity for the kernel functions, see \citet{Mohri} for the derivation. The second term is that the supremeum with respect to $b$ is that if \( \sum \sigma_i \)  is positive, then  \( b = \alpha\Lambda + 1 \), if  \( \sum \sigma_i \) is negative, then  \( b = -(\alpha\Lambda + 1) \).

Thus, from the Massart's lemma, where the hypothesis set is 
\[
\{ (\alpha\Lambda + 1),\, -(\alpha\Lambda + 1) \} \subset \mathbb{R}^1
\]
Thus, by setting \( A \coloneqq \{ (\Lambda + 1),\, -(\Lambda + 1) \} \subset \mathbb{R} \) and this results in
\[
\frac{1}{n} \mathbb{E}_\sigma \left[ \sup_{|b| \leq \Lambda + 1} \sum_{i=1}^n \sigma_i b \right]
= \frac{1}{n} \mathbb{E}_\sigma \left[ \sup_{z \in A} \sum_{i=1}^n \sigma_i z_i \right]
\leq \frac{\sqrt{2 \log 2} (\Lambda + 1)}{\sqrt{n}}
\leq \frac{2 (\alpha\Lambda + 1)}{\sqrt{n}}
\]
In conclusion, we have
\begin{align}\label{app_rade_gen}
\hat{\mathfrak{R}}_S(\mathcal{F}) \leq \frac{3 \alpha\Lambda + 2}{\sqrt{n}}    .
\end{align}
This concludes the proof of  Corollary~\ref{thm_kernel_r1_square_gen}.

Next, we show that any solution $f_\lambda$ to the regularized ERM problem must have a norm bounded by $\lambda$. Consider the objective function $J(f) = L_n(f) + \lambda \|f\|_{\mathcal{H}}^2$. By definition, the optimal solution $f_\lambda=f_\lambda'+b$ with $f'_\lambda\in \calh$ and $b\in\real$ must satisfy:
\begin{align}
    L_n(f_\lambda)+\lambda \|f_\lambda\|^2_\mathcal{H}=\inf_{f \in \mathcal{H}} L_n(f)+\lambda \|f\|^2_\mathcal{H}
\end{align}
The infimum must be less than or equal to the objective value for any function, including the zero function $f=0$. Therefore:
\begin{align}
\inf_{f \in \mathcal{H}} L_n(f)+\lambda \|f\|^2_\mathcal{H}  \le L_n(0) + \lambda \|0\|^2_\mathcal{H} = L_n(0).
\end{align}

Since the loss is non-negative ($L_n(f_\lambda) \ge 0$), we have:
$$\lambda \|f_\lambda\|_{\mathcal{H}}^2 \le L_n(f_\lambda) + \lambda \|f_\lambda\|_{\mathcal{H}}^2 \le L_n(0)$$
This yields the desired relationship: $\|f_\lambda\|_{\mathcal{H}} \le \sqrt{L_n(0) / \lambda}$. Thus, the norm of the optimal solution is of the order $\mathcal{O}(1/\sqrt{\lambda})$. Note that $L_n(0)\leq 1$.

Based on this, we plug this $\lambda$-dependent norm bound into the standard formula for the Rademacher complexity of a function class in an RKHS. The complexity for a class of functions $\{f+b \mid \|f\|_{\mathcal{H}} \le \alpha, |b| \le \beta\}$ is bounded by the sum of the complexities for the kernel part and the bias part. And $\alpha$ corresponds to $\sqrt{L_n(0) / \lambda}$. As for $\beta$, let us recall the definition of the sqaured loss
\begin{align}
    L_n(f)&=\sum_i\frac1n(y_i-(\sum_ja_jk(x_i,x_j)+b))^2+\lambda\|f\|_\calf^2\\
    &=\sum_{i:y_i=1}\frac1n(1-(\sum_ja_jk(x_i,x_j)+b))^2+\lambda\|f\|_\calf^2+\sum_{i:y_i=0}\frac1n(0-(\sum_ja_jk(x_i,x_j)+b))^2+\lambda\|f\|_\calf^2
\end{align}
From the above relation, we can see that $|b|\leq \alpha\Lambda+1$ is the region where the optimal solution exists.

When using ERM solution, by substituting $\alpha=\sqrt{L_n(0)/\lambda}$ into the estimate of the Rademacher complexity in Eq.~\eqref{app_rade_gen}, we obtain the result.
\end{proof}

\subsection{Proof of Theorem~\ref{thm_composite_higher_dim}}
\begin{proof}
From Proposition 7 in \citet{rakhlin19a}, the RKHS associated with the Laplace kernel on $\real^d$ is the Sobolev space $H^{s} = W^{2,s}$ with $s = (d + 1)/2$ (see also \citet{buchholz2022kernel}).
According to Theorem~7 in \citet{bourdaud2022introduction}, for given functions $f$ and $r$, we have $r \circ f \in H^{s}$ if and only if $r \in H^{s}$. However, here we take $r$ from a Lipschitz function class, so  $r$ is not generally in $H^{s}$, and thus the result follows.

Similarly, $k \in \calk_1$ implies $k \in H^{1}$ but not $k \in H^{s}$. Therefore, we obtain the result.
\end{proof}

\subsection{Proof of Theorem~\ref{theorem_universal__kernel}}

\begin{proof}
We use the approximation theory given in Corollary 5.29 of \citet{steinwart2008support}, which states that for a continuous Nemitski loss function $\ell$, the Bayes risk and the minimum achievable risk within the RKHS are equivalent. According to Definition 2.16 in \citet{steinwart2008support}, the squared loss with $L_2$ regularization is a Nemitski loss function. Then, Corollary 5.29 implies that
\begin{align}
\inf_{f} L(f) = \inf_{f \in \calf} L(f),
\end{align}
where the infimum on the left-hand side is taken over all measurable functions from $\calx \to \real$. This equivalence follows from the approximation power of universal kernels; see the proof of Corollary 5.29 in \citet{steinwart2008support} for details.

With this in mind, and following the argument in the proof of Claim 5.1 in \citet{blasiok2023when}, we upper bound the post-processing gap as follows. By the definition of the infimum,
\[
f^* = \argmin_{f \in \mathcal{F}} \Ex \ell_{\text{sq}}(f(X),Y)  + \lambda \|f\|_\calf^2.
\]
Consider the solution \( f^* \) and an arbitrary \( h \in \lip_{L=1} \). Note that \( f^* + h \circ f^* \) is a measurable function. For notational simplicity, we denote \( r \circ f^* \coloneqq f^* + h \circ f^* \).

\begin{align}
     &\Ex \ell_{\text{sq}}(f_n^*(X),Y)-\Ex \ell_{\text{sq}}(r\circ f_n^*(X),Y)\\
     &=\Ex \ell_{\text{sq}}(f_n^*(X),Y))+\lambda\|f\|_\calf^2-\lambda\|f\|_\calf^2-\Ex \ell_{\text{sq}}(r\circ f_n^*(X),Y)\\
     &\leq L(f_n^*)-L(f^*)+L(f^*)-L(r\circ f_n^*)+\lambda\\
     &\leq L(f^*)+\mathrm{err}_{\mathrm{ex}}(n)-L(r\circ f_n^*)+\lambda\\
     &\leq \lambda+\mathrm{err}_{\mathrm{ex}}(n)
\end{align}
In the last inequality, we used the fact that
\begin{align}
    L(f^*)-L(r\circ f_n^*)=\inf_{f\in\calf}L(f)-L(r\circ f_n^*)=\inf_{f}L(f)-L(r\circ f_n^*)\leq 0
\end{align}
Since infimum is taken with respect to all measurable functions, therefore $\inf_{f}L(f)\leq L(r\circ f_n^*)$ holds.

Next we upper bound $\mathrm{err}_{\mathrm{ex}}(n)$. This is well studied in the literature of the generalization analysis and from  Proposition 4.1 in \citet{Mohri}, we have
\begin{align}
L(f^*_n)-\inf_{f\in\calf}L(f) \leq \mathrm{err}_{\mathrm{ex}}(n)\leq 2\sup_{f\in\calf}|L(f)-L_n(f)|\leq 2\left( 2\mathfrak{R}_{\mathcal{D},n}(\mathcal{F})+\sqrt{\frac{\log\frac{2}{\delta}}{2n}}\right).
\end{align}

Then we have
\begin{align}
\smce(f^*_n,\cald) \leq \sqrt{\lambda + 4 \mathfrak{R}_{\mathcal{D},n}(\mathcal{F})+\sqrt{2\frac{\log\frac{2}{\delta}}{n}}}.
\end{align}

\end{proof}

\subsection{Proof of Corollary~\ref{thm_erm_dual}}
We first upper bound the training dual smooth CE. The proof is almost identical to that of Theorem~\ref{thm_erm_smooth}. We first introduce the training dual smooth CE as
\textbf{\begin{align}
    \smce^{(\psi,1/4)}(g, S) \coloneqq \sup_{h \in \mathrm{Lip}_{1/4}(\mathbb{R}, [-1,1])} \frac{1}{n}\sum_{i=1}^n\left[ h(g(X_i)) \cdot (Y_i - f(X_i)) \right].
\end{align}}
We also define the empirical counterpart of the dual post-processing gap as
\begin{align}
    \pgap^{(\psi,1/4)}(g, S) &\coloneqq  \frac{1}{n}\sum_{i=1}^n \ell^\psi(g(X_i), Y_i) - \inf_{h \in \mathrm{Lip}_{L=1}(\real,[-4,4])}  \frac{1}{n}\sum_{i=1}^n \ell^\psi(f(X_i)+h(g(X_i)), Y_i)]
\end{align}
We can prove that
\begin{align}\label{app_proof_dual_smooth_empirical}
2 \smce^{(\psi,1/4)}(g, S)^2 \leq \pgap^{(\psi,1/4)}(g, S) \leq 4 \smce^{(\psi,1/4)}(g, S).
\end{align}
The proof of this is exactly the same as that of Lemma 4.7 in \citet{blasiok2023when}, where we simply replace the expectation by $\cald$ with that of the empirical expectation. 

To simplify the notation, we express $r\circ g \coloneqq g + h \circ g$. Then we will upper bound the training 
\begin{align}
&\frac{1}{n}\sum_{i=1}^n\ell^\psi(g(X_i), Y_i)-\inf_{h \in \mathrm{Lip}_{L=1}(\real,[-4,4])}\frac{1}{n}\sum_{i=1}^n\ell^\psi(r\circ g(X_i),Y_i)\\
&=\frac{1}{n}\sum_{i=1}^n\ell^\psi(g(X_i), Y_i)+\lambda\|g\|_\calg^2-\lambda\|g\|_\calg^2-\inf_{h \in \mathrm{Lip}_{L=1}(\real,[-4,4])}\frac{1}{n}\sum_{i=1}^n\ell^\psi(r\circ g(X_i),Y_i)\\
&=\frac{1}{n}\sum_{i=1}^n\ell^\psi(g(X_i), Y_i)+\lambda\|g\|_\calg^2-\lambda\|g\|_\calg^2\\
&\qquad-\inf_{h \in \mathrm{Lip}_{L=1}(\real,[-4,4])}\left(\frac{1}{n}\sum_{i=1}^n\ell^\psi(r\circ g(X_i),Y_i)+\lambda\|r\circ g\|_\calg^2-\lambda\|r\circ g\|_\calg^2\right)\\
&=L_n(g)-\inf_{h \in \mathrm{Lip}_{L=1}(\real,[-4,4])}\left(\frac{1}{n}\sum_{i=1}^n\ell^\psi(r\circ g(X_i),Y_i)+\lambda\|r\circ g\|_\calg^2+\lambda\|g\|_\calg^2-\lambda\|r\circ g\|_\calg^2\right)\\
&\leq L_n(g)-\inf_{h \in \mathrm{Lip}_{L=1}(\real,[-4,4])}\left(\frac{1}{n}\sum_{i=1}^n\ell^\psi(r\circ g(X_i),Y_i)+\lambda\|r\circ g\|_\calg^2\right)\\
&\qquad-\inf_{h \in \mathrm{Lip}_{L=1}(\real,[-4,4])}\left(\lambda\|g\|_\calg^2-\lambda\|r\circ g\|_\calg^2\right)\\
&\leq L_n(g^*_n)+\mathrm{err}_n(g)-\inf_{h \in \mathrm{Lip}_{L=1}(\real,[-4,4])}\left(\frac{1}{n}\sum_{i=1}^n\ell^\psi(r\circ g(X_i),Y_i)+\lambda\|r\circ g\|_\calg^2\right)+\lambda G^2\\
&=L_n(g^*_n)+\mathrm{err}_n(g)-\inf_{h \in \mathrm{Lip}_{L=1}(\real,[-4,4])}L_n(r\circ g)+\lambda G^2\\
&\leq \mathrm{err}_n(g)+\lambda G^2
\end{align}
where we used that $L_n(f)=L_n(f^*_n)+\mathrm{err}_n(f)$ and $\|r\circ g\|_\calg^2-\|g\|_\calg^2\leq G$.
Moreover, we used the relation
\begin{align}
 L_n(g^*_n)\leq \inf_{h \in \mathrm{Lip}_{L=1}(\real,[-4,4])}L_n(r\circ g)
\end{align}
in the last line. Then, using Eq.~\eqref{app_proof_dual_smooth_empirical}, we have the upper-bound for the training dual smooth CE.

Next, we study the generalization error for $\smce^{(\psi,1/4)}(g, \mathcal{D})$

\begin{align}
    &\smce^{(\psi,1/4)}(g, \mathcal{D})-\smce^{(\psi,1/4)}(g, S)\\
    &=\sup_h \Ex h(g(X))\cdot (Y-f(X))-\sup_{h'}\frac{1}{n}\sum_i h'(g(X_i))\cdot (Y_i-f(X_i))\\
    &=\sup_h\Ex\left[\frac{1}{2}(Y-(f-h(g(X)))^2-\frac{1}{2}(Y-f)^2-\frac{1}{2}h(g(X))^2\right]\\
    &-\sup_{h'}\frac{1}{n}\sum_i\left[\frac{1}{2}(Y_i-(f_i-h'(g(X_i)))^2-\frac{1}{2}(Y-f_i)^2-\frac{1}{2}h'(g(X_i))^2\right]\\
    &\leq \sup_\eta\Big( \Ex\left[\frac{1}{2}(Y-(f-h(g(X)))^2-\frac{1}{2}(Y-f)^2-\frac{1}{2}h(g(X))^2\right]\\
    &-\frac{1}{n}\sum_i\left[\frac{1}{2}(Y_i-(f_i-h(g(X_i)))^2-\frac{1}{2}(Y-f_i)^2-\frac{1}{2}h(g(X_i))^2\right]\Big)
\end{align}

Then we proceed the proof exactly in the same way as Appendix~\ref{sec_proof_rademacher}.

\subsection{Proof of Corollary~\ref{thm_kernel_r1_logistic_gen} and \ref{col_klr}}
By the assumptions, we can apply Corollary~\ref{thm_erm_dual} to this setting. All we need is to estimate the Rademacher complexity.

We then define the set of functions obtained by $\sigma(g)$ for any $g\in\calg$ as $\calf$, and $\sigma$ is $1/4$ Lipshitz function, 
\begin{align}
    \hat{\mathfrak{R}}_S(\mathcal{F})\leq \frac{1}{4}\hat{\mathfrak{R}}_S(\mathcal{G})
\end{align}
and $\hat{\mathfrak{R}}_S(\mathcal{G})$ can be bounded exactly in the same way as the proof of Corollary~\ref{thm_kernel_r1_square_gen} in Appendix~\ref{app_proof_kernel}.

Next, we show that any solution $g_\lambda$ to the regularized ERM problem must have a norm bounded by $\lambda$. Consider the objective function $J(g) = L_n(g) + \lambda \|g\|_{\mathcal{H}}^2$. This yields the desired relationship: $\|g_\lambda\|_{\mathcal{G}} \le \sqrt{L_n(0) / \lambda}$.

$\len(v, y) \coloneqq -y\log v - (1 - y)\log(1 - v)$.

Based on this, we plug this $\lambda$-dependent norm bound into the standard formula for the Rademacher complexity of a function class in an RKHS. The complexity for a class of functions $\{f+b \mid \|f\|_{\mathcal{H}} \le \alpha, |b| \le \beta\}$ is bounded by the sum of the complexities for the kernel part and the bias part. And $\alpha$ corresponds to $\sqrt{L_n(0) / \lambda}$. As for $\beta$, let us recall the definition of the sqaured loss
\begin{align}L_n(g)&=\sum_{i:y_i=1}\frac1n\log (1+e^{-(g(x_i)+b)})+\sum_{i:y_i=0}\frac1n\log (1+e^{g(x_i)+b}))+\lambda\|g\|_\calg^2
\end{align}
From the above relation, we can see that $|b|\leq \alpha\Lambda+1$ is the region where the optimal solution exists. Note that $L_n(0)\leq \log 2 <1$.

Here we also present the result that corresponds to Theorem~\ref{theorem_universal__kernel}:
\begin{theorem}
Let $k$ be a universal kernel with associated RKHS $\calh$. Let $\calg = \calh \oplus \mathbb{R} = \{g + b \mid g \in \calh, b \in \mathbb{R}\}$. Suppose there exist constant $\Lambda$ such that $\sup_{x, x' \in \mathcal{X}} k(x, x') \leq \Lambda$. Then, with probability at least $1 - \delta$ over the draw of the training dataset, it holds that
\begin{align}
\smce^{(\psi,1/4)}(g^*_n, \mathcal{D})\leq \sqrt{\lambda G^2 + \frac{3\Lambda + 2}{\sqrt{\lambda n}} + \sqrt{\frac{2\log\frac{2}{\delta}}{n}}}.
\end{align}
\end{theorem}
The proof of this theorem is almost identical to that of Theorem~\ref{theorem_universal__kernel}, since the logistic loss is also a continuous Nemitski loss and thus, we can apply the same techniques.

\section{Additional discussion}\label{app_metrics_proper}

\subsection{Post processing gap and Calibration metrics}
Following \citet{blasiok2023when}, we introduce the general proper loss function and its relation to the post-processing gap.

A proper loss function $\ell$ can always be represented using a convex function $\phi$ as follows:
\begin{align}
    \ell(p, y) = -\phi(p) - \nabla \phi(p) \cdot (y - p),
\end{align}
where $\phi: [0,1] \to \mathbb{R}$ is a convex function and $\nabla \phi(p)$ denotes a subgradient at $p$. Following \citet{blasiok2023when}, we assume that $\phi$ is differentiable.

We define the convex conjugate of the function $\phi(p)$ as follows: for all $s \in \mathbb{R}$,
\begin{align}
    \psi(s) = \sup_{p \in [0,1]} \left\{ s \cdot p - \phi(p) \right\}.
\end{align}
The \emph{dual loss} $\ell^\psi : \mathbb{R} \times \mathcal{Y} \to \mathbb{R}$ is then defined as
\begin{align}
    \ell^\psi(s, y) \coloneqq \psi(s) - s \cdot y.
\end{align}
By Fenchel–Young duality, this relationship is inverted as $p = \nabla \psi(s)$, and with these definitions, the proper loss can equivalently be written as
\begin{align}
    \ell(p, y) = \ell^\psi(\nabla \phi(p), y) = \psi(\nabla \phi(p)) - \nabla \phi(p) \cdot y.
\end{align}

We remark that the relation $p = \nabla \psi(s)$ can be interpreted as mapping logits to predicted probabilities. For details and proofs, see \citet{blasiok2023when}.

\citet{blasiok2023when} considered modeling the score function $s$ by a function $g$, and then applying the transformation $p = \nabla \psi(s)$. Thus, they proposed to apply post-processing to $g$, which leads to the following definition:
\begin{definition}[Dual post-processing gap]
Assume that $\psi$ is a differentiable and convex function with derivative $\nabla \psi(t) \in [0,1]$ for all $t \in \real$, and that $\psi$ is $\lambda$-smooth. Given $\psi$, $\ell^\psi$, $g:\calx \to \real$, and distribution $\cald$, we define the dual post-processing gap as
\begin{align}
    \pgap^{(\psi,\lambda)}(g,\cald) \coloneqq \Ex[\ell^\psi(g(X),Y)] - \inf_{h \in \lip_{1}(\real,[-1/\lambda,1/\lambda])} \Ex[\ell^\psi(g(X)+h(g(X)),Y)].
\end{align}
\end{definition}

When considering the cross-entropy loss, the dual post-processing gap corresponds to improving the logit function.

\begin{definition}[Dual smooth calibration]
Consider the same setting as in the definition of the dual post-processing gap. Given $\psi$ and $g$, define $f(\cdot) = \nabla \psi(g(\cdot))$. The dual calibration error of $g$ is defined as
\begin{align}
    \smce^{(\psi,\lambda)}(g,\cald) \coloneqq \sup_{h \in \lip_{L=\lambda}(\real,[-1,1])} \Ex[\eta(g(X)) \cdot (Y - f(X))].
\end{align}
\end{definition}

Then, similarly to the relationship between the smooth ECE and the post-processing gap, the following holds: if $\psi$ is a $\lambda$-smooth function, then
\begin{align}
    \frac{1}{2} \smce^{(\psi,\lambda)}(g,\cald)^2 \leq \frac{\lambda}{2} \pgap^{(\psi,\lambda)}(g,\cald)^2 \leq \smce^{(\psi,\lambda)}(g,\cald),
\end{align}
and
\[
\smce(f,\cald) \leq \smce^{(\psi,\lambda)}(g,\cald)
\]
also holds. Thus, by studying the dual post-processing gap, we can obtain bounds on the smooth calibration error. By considering $L_2$-regularized objective function $\Ex[\ell^\psi(g(X),Y)]+\|g\|_\calg^2$ and its empirical counterpart, we can develop the theory for the general dual smooth CE and ERM in a similar way to the case of the squared and cross-entropy loss.

\subsection{Relationships different calibration metrics}
\citet{blasiok2023unify} introduced the ground truth distance for calibration, defined as follows:
\begin{definition}[True distance to calibration]
We define the true distance of a predictor $f$ from calibration as
\begin{align}
\dCE \coloneqq \inf_{g \in \calcald} \Ex_{\cald} |f(x) - g(x)|,
\end{align}
where $\calcald$ denotes the set of predictors that are perfectly calibrated with respect to $\cald$.
\end{definition}

This provides an ideal notion for measuring calibration; see \citet{blasiok2023unify} for details. They showed that the smooth CE both upper and lower bounds the true distance to calibration:
\begin{align}
\smce(f,\cald) \leq \dCE \leq 4\sqrt{2 \smce(f,\cald)}.
\end{align}

On the other hand, the commonly used ECE, defined as
\begin{align}
\ece_\cald(f) \coloneqq \Ex_\cald \left[ \left| \Ex_\cald[y | f(x)] - f(x) \right| \right],
\end{align}
is discontinuous, and \citet{blasiok2023unify} showed that ECE does not lower bound $\dCE$ unless continuity of the conditional expectation is assumed.

\citet{blasiok2023unify} also established the relationship between $\dCE$ and the binned ECE. Given a partition $\cali = \{I_1, \dots, I_m\}$ of $[0,1]$ into intervals, the binned ECE is defined as
\begin{align}
\binece_\cald(f, \cali) = \sum_{j \in [m]} \left| \Ex[(f - y)\mathbbm{1}(f \in I_j)] \right|.
\end{align}
They showed that by adding the bin widths and minimizing over the choice of partition, we obtain the following definition:
\begin{align}
\text{intCE}(f) \coloneqq \min_{\cali} \left( \binece_\cald(f, \cali) + w(\cali) \right),
\end{align}
where
\begin{align}
w(\cali) \coloneqq \sum_{j \in [m]} \left| \Ex w(I_j) \mathbbm{1}(f \in I_j) \right|,
\end{align}
and $w(I)$ denotes the width of interval $I$. Then, the following bound holds (Theorem 6.3 in \citet{blasiok2023unify}):
\begin{align}
\dCE \leq \text{intCE}(f) \leq 4\sqrt{\dCE}.
\end{align}

As we have seen, bounding the smooth CE leads to a bound on $\dCE$, which in turn bounds $\text{intCE}(f)$, which corresponds to the binned ECE, which is optimized with respect to the partition.

\section{Details of experimental settings}
\label{app:detail_exp}
In this section, we summarize the detail information of our numerical experiments in Section~\ref{sec:experiments}.
Our experiments were conducted on NVIDIA GPUs with $32$GB memory (NVIDIA DGX-1 with Tesla V100 and DGX-2).
The source code to reproduce all experiments is available at \url{https://github.com/msfuji0211/erm_calibration}.

\begin{table}[th]
\centering
\caption{Datasets used in our experiments}
\label{table:datasets}
\begin{tabular}{lcccc}
\toprule
\textbf{Dataset} & \textbf{Classes} & \textbf{Train data ($\ntr$)} & \textbf{Recalibration data ($\nre$)} & \textbf{Test data ($\nte$)} \\
\midrule \midrule
KITTI & $2$ & $16000$ & $1000$ & $8000$ \\
PCam & $2$ & $22768$ & $1000$ & $9000$ \\
\bottomrule
\end{tabular}
\end{table}

\subsection{Toy data experiments (\texorpdfstring{$\calx=\mathbb{R}$}{Lg})}
\label{app:detail_exp_synthesize}
To investigate the behavior of different kernel-based methods under controlled conditions, we first conduct experiments on synthetic two-dimensional binary classification tasks. These toy experiments serve to isolate and visualize model behavior with respect to classification performance and calibration quality, without the confounding factors present in real-world datasets.

\textbf{Data generation.}
We generate synthetic data using a simple but structured stochastic process. For each of $n$ samples, a binary label $y \in \{0, 1\}$ is drawn independently from a Bernoulli$(0.5)$ distribution. Given the label, the input feature $x \in \mathbb{R}^2$ is sampled from a Gaussian distribution centered at $\mu_1 = [-1, -1]^T$ for $y = 1$, and at $\mu_0 = [1, 1]^T$ for $y = 0$, with identity covariance $\Sigma = I_2$ in both cases. That is,
\[
x \mid y = 1 \sim \mathcal{N}([-1, -1]^T, I), \quad x \mid y = 0 \sim \mathcal{N}([1, 1]^T, I).
\]
This construction induces a smooth but nonlinear Bayes decision boundary, suitable for evaluating kernel methods.

\textbf{Models and kernels.}
We evaluate two models:
\begin{itemize}
    \item \textbf{KRR}: Kernel Ridge Regression with theoretically motivated $\lambda_n = n^{-1/2}$ for Gaussian kernels and $\lambda_n = n^{-1/3}$ for Laplace kernels.
    \item \textbf{KLR}: Kernel Logistic Regression optimized via gradient descent.
\end{itemize}
Each model is evaluated using two kernels: the Gaussian kernel $k(x, x') = \exp(-\|x - x'\|^2 / 2\sigma^2)$ and the Laplace kernel $k(x, x') = \exp(-\|x - x'\| / \sigma)$. For each kernel, the bandwidth $\sigma$ is selected using the median heuristic on the training data.

\textbf{Metrics.}
We assess both accuracy and calibration using the following metrics:
\begin{itemize}
    \item \textbf{Kernel Calibration Error (KCE)}: Evaluated with both Gaussian and Laplace kernels, with $\sigma$ determined by a heuristic on the predicted confidence vector.
    \item \textbf{Smooth Calibration Error (SCE)}: A continuous variant of calibration error designed for better sample efficiency.
    \item \textbf{Expected Calibration Error (ECE)}: Classical binning-based calibration metric with the number of bins set to $\lfloor n^{1/3} \rfloor$ for $n$ data points, following \citet{futami2024information, fujisawa2025}.
\end{itemize}

\textbf{Experimental protocol.}
We evaluate performance as a function of training set size, with $n_{\text{train}}$ logarithmically spaced from 100 to 10,000. For each setting, experiments are repeated with 10 different random seeds for robustness. We also evaluate sensitivity to the regularization parameter $\lambda$ by fixing $n_{\text{train}} = 10{,}000$ and varying $\lambda$ over a logarithmic grid from $10^{-4}$ to $10^{2}$.

\textbf{Implementation.}
All methods are implemented using PyTorch. Gradient descent for KLR is run for up to 1,000 iterations with a step size of $0.01$ and stopping tolerance of $10^{-6}$. Results are reported on both training and test sets. Each experiment logs the metrics above and saves results in a CSV format for post-hoc statistical analysis.

\subsection{Recalibration experiments (\texorpdfstring{$\calx=\mathbb{R}$}{Lg})}
\label{app:detail_exp_recalibration}
We provide the details of the datasets along with the number of training, recalibration, and test samples in Table~\ref{table:datasets}. For the models, we used XGBoost~\citep{Chen16}, Random Forests~\citep{breiman01}, and a one-layer neural network (NN) for the KITTI and PCam experiments.
All experiments—including the training of XGBoost, Random Forests, and the one-layer NN—were conducted using code adapted from \citet{wenger20a}~\footnote{\url{https://github.com/JonathanWenger/pycalib}}.

\textbf{Performance evaluation:}
We evaluated predictive accuracy and binned ECE, using $B = \lfloor n_{\mathrm{re}}^{1/3} \rfloor$, in accordance with the theoretical insights from \citet{futami2024information, fujisawa2025}.
Additionally, we included two other calibration metrics: KCE and SCE.
To train the recalibration functions, we performed 10-fold cross-validation and reported the mean and standard deviation of both performance metrics.

\subsection{Real dataset experiments (Binary classification benchmarks; \texorpdfstring{$\calx=\mathbb{R}^{d}$}{Lg})}
\label{app:detail_exp_realdata}
We perform binary classification experiments using real-world tabular datasets to evaluate calibration and generalization performance across various kernel methods and sample sizes. Two separate protocols are employed:

\textbf{(A) Sample size variation experiment.}
This setting aims to evaluate how calibration performance evolves with increasing sample size. We consider the following methods: (i) KRR and (ii) KLR, each with either an RBF or Laplace kernel. For scalability, we use random Fourier features (RFF) for KLR. The Laplace kernel is approximated via a variant of RFF using samples from a Cauchy distribution. The corresponding feature mapping is implemented in our \texttt{LaplaceSampler} class.

For each dataset, we split the data into train/test with an 80/20 ratio while maintaining class balance. For training, we apply stratified subsampling of size $n \in \{50, \ldots, 2000\}$ (log-spaced, with 10 candidates). Each experiment is repeated with 5 different random seeds. The regularization hyperparameters are fixed as follows: $\alpha = 0.1$ for KLR and $\alpha = n^{-1/2}$ (RBF) or $\alpha = n^{-1/3}$ (Laplace) for KRR, based on empirical performance.

Bandwidth parameters for both kernels are selected via the median heuristic: for the RBF kernel, $\gamma = \frac{1}{2 \sigma^2}$; for the Laplace kernel, $\gamma = \frac{1}{\sigma}$, where $\sigma$ is the median pairwise Euclidean distance among training samples.

\textbf{(B) Regularization parameter variation.}
To assess sensitivity to the regularization hyperparameter, we fix the training set size at $n=2000$ and vary $\alpha$ over a logarithmic grid: $\alpha \in \{10^{-4}, \ldots, 10^{2}\}$. The same model families are considered as in (A), using fixed kernel parameters ($\gamma=0.1$ for all models) to isolate the impact of $\alpha$.

\textbf{Evaluation metrics.}
We report three calibration metrics: ECE with optimal bins~\citep{futami2024information, fujisawa2025}, smoothed ECE, and MMCE. For KRR, probabilities are obtained by clipping regression outputs to $[10^{-6}, 1 - 10^{-6}]$ for stability.

\textbf{Datasets.}
We use six binary classification datasets from OpenML: \texttt{kr-vs-kp}, \texttt{spambase}, \texttt{sick}, \texttt{churn}, and \texttt{Satellite}. Features are standardized after applying appropriate imputation and one-hot encoding using scikit-learn pipelines. All preprocessing steps are fit only on the training set to avoid data leakage.

\textbf{Reproducibility.}
All experiments are implemented in Python using scikit-learn and CVXPY. Stratified sampling ensures class balance in subsamples. The full experimental code and data generation scripts will be made available upon publication.

\section{Additional experimental results}
\label{app:additional_res}
In this section, we present additional experimental results.
Figures~\ref{fig:res_recab_all_kitti}–\ref{fig:res_recab_all_pcam} show the complete results of our recalibration experiments described in Section~\ref{subsec:exp_recalibration}.
Consistent with our theoretical analysis, we observe that increasing the regularization parameter $\lambda$ leads to higher smooth CE for both Laplace and Gaussian kernels, reflecting the expected effect of stronger regularization.
Conversely, increasing the recalibration sample size $n_{\mathrm{re}}$ consistently lowers the smooth CE, demonstrating the anticipated convergence behavior.
These findings highlight the practical applicability of our theory to real-world recalibration scenarios.

We further show our experimental results on some real-world datasets explained in Appendix~\ref{app:detail_exp_realdata} in Figures~\ref{fig:res_real_n} and \ref{fig:res_real_lambda}.
Similarly, we observe that setting the regularization parameter $\lambda$ too small or too large results in unstable smooth CE values for both Laplace and Gaussian kernels.
In contrast, increasing the recalibration sample size $n_{\mathrm{re}}$ consistently reduces the smooth CE in most cases, exhibiting convergence behavior aligned with our theoretical results.
These findings further support the reliability of our theory, demonstrating its applicability to real-world binary classification tasks.

\begin{figure}[t]
    \centering
    \includegraphics[width=\textwidth]{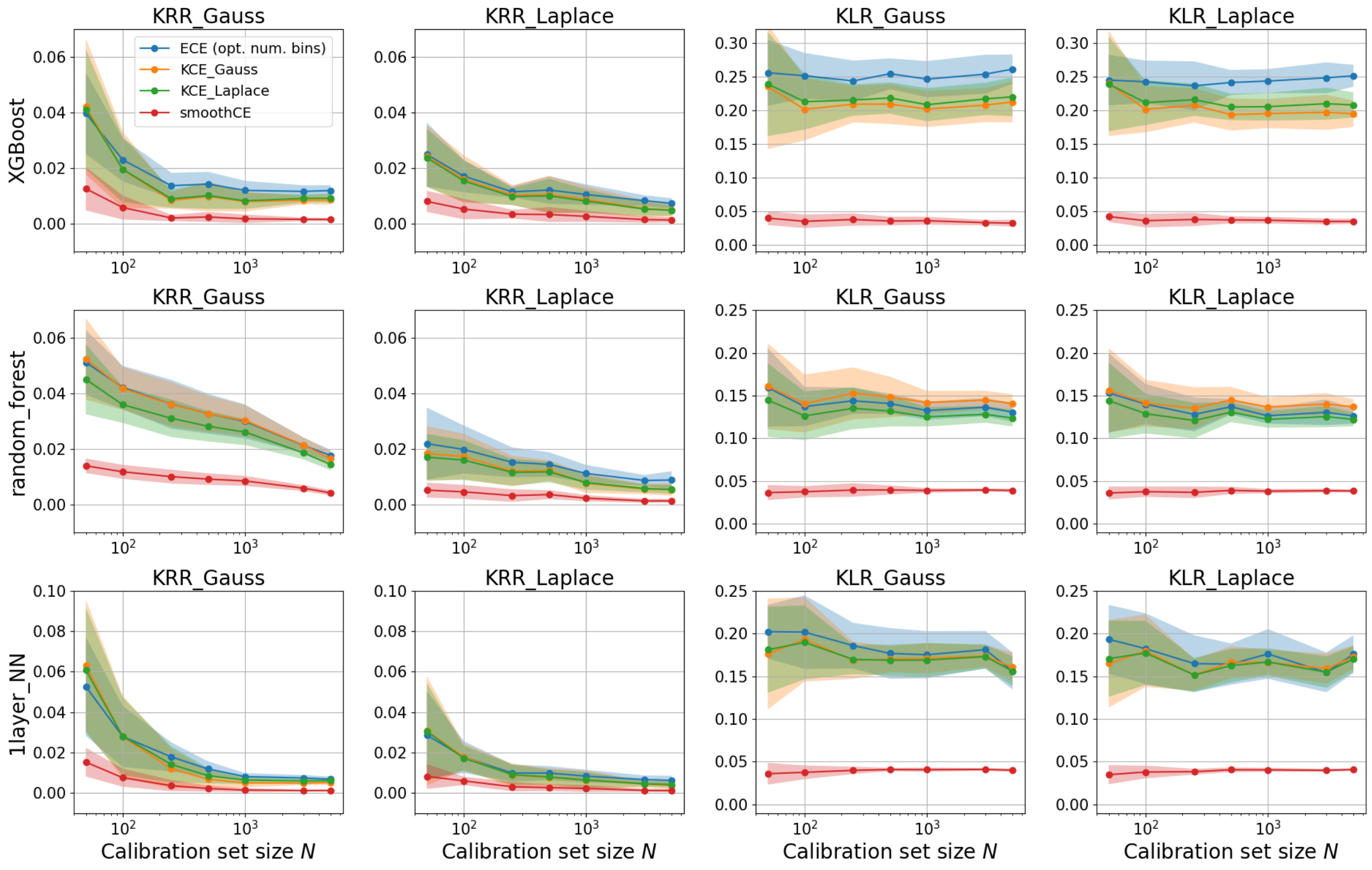}
    \caption{All Experimental Results of Recalibration: Effect of Recalibration Sample Size on Calibration Metrics on the KITTI dataset.}
    \label{fig:res_recab_all_kitti}
\end{figure}

\begin{figure}[t]
    \centering
    \includegraphics[width=\textwidth]{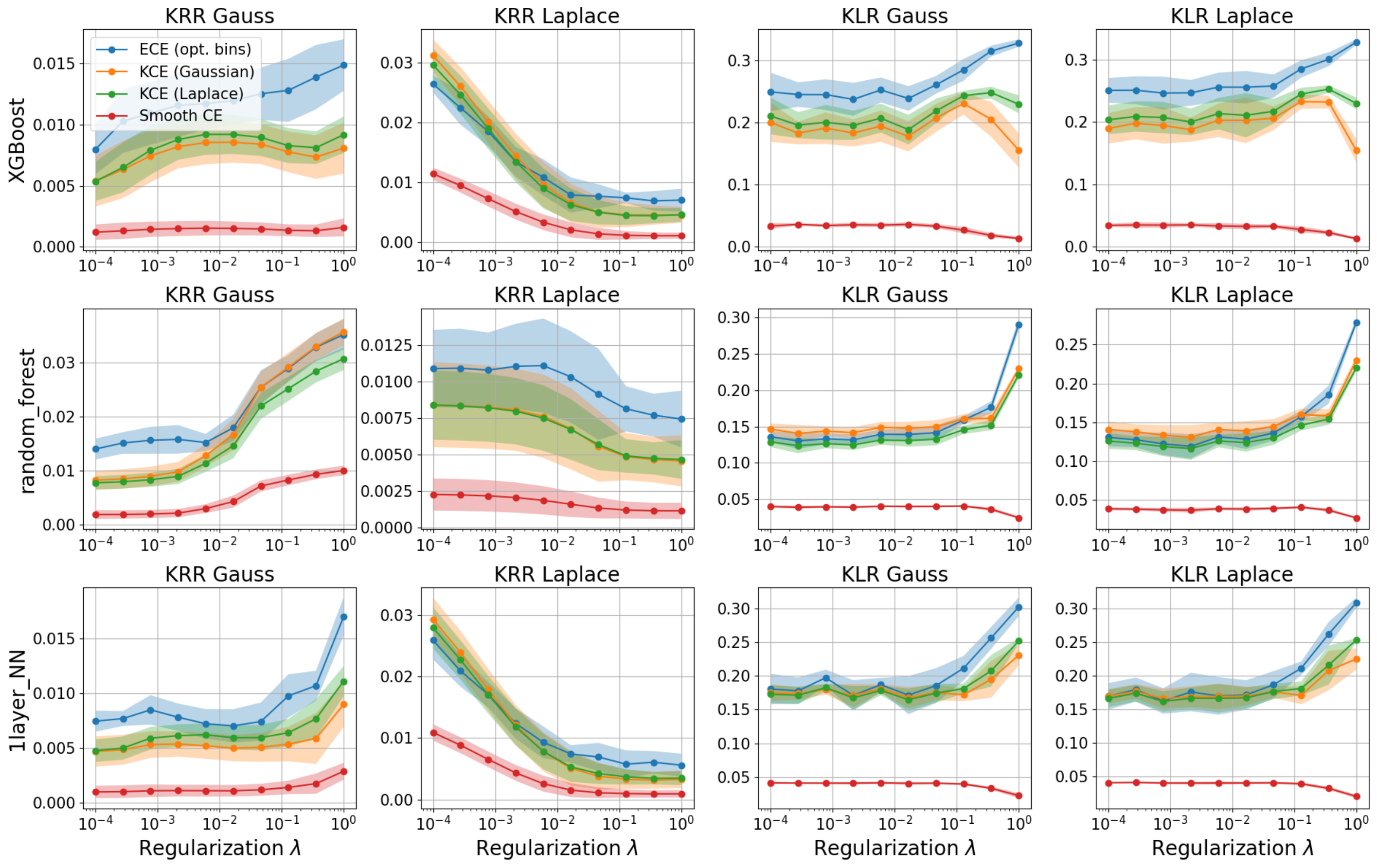}
    \caption{All Experimental Results of Recalibration: Effect of Regularization parameter $\lambda$ on Calibration Metrics on the KITTI dataset.}
    \label{fig:res_recab_all_kitti_lam}
\end{figure}

\begin{figure}[t]
    \centering
    \includegraphics[width=\textwidth]{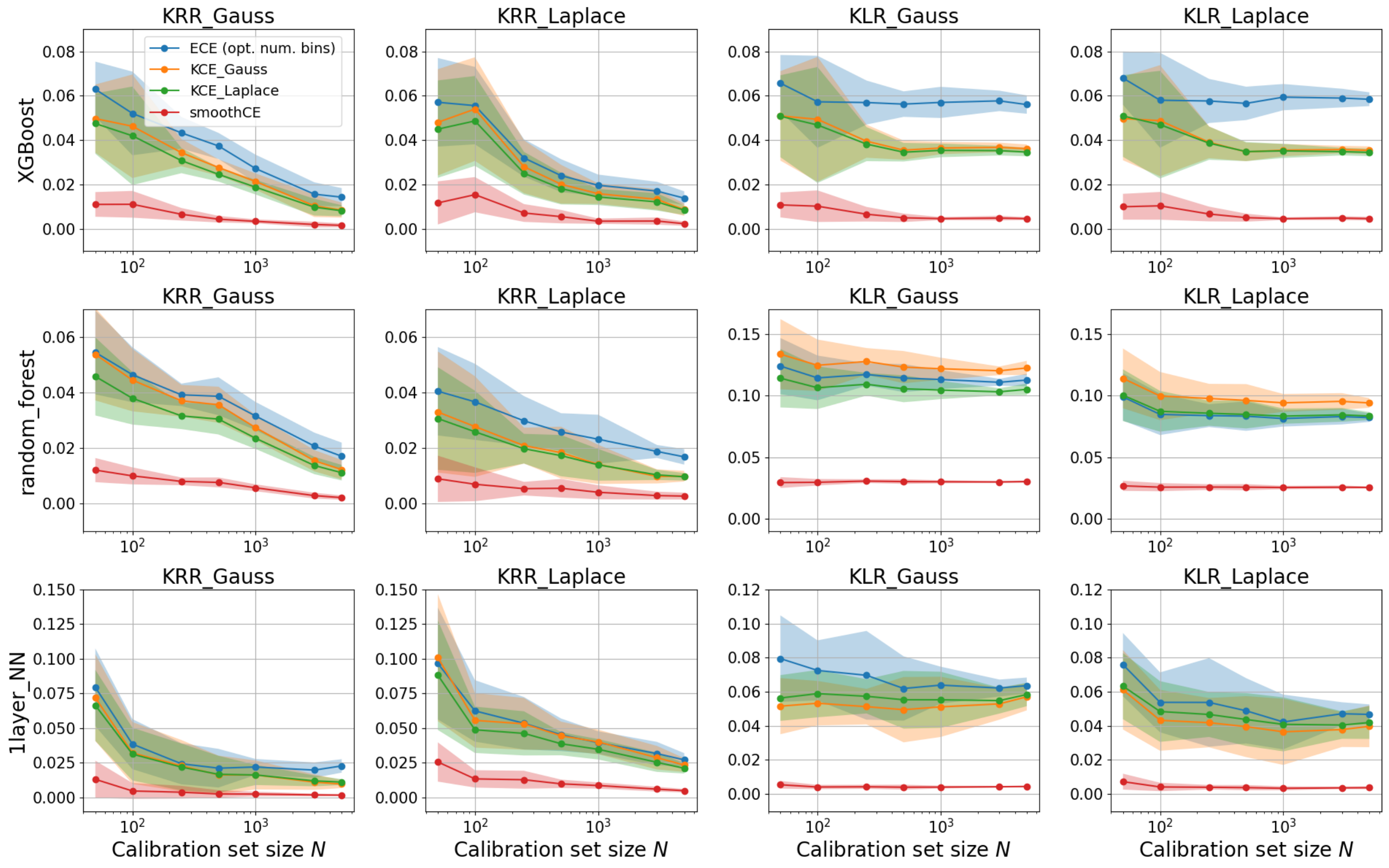}
    \caption{All Experimental Results of Recalibration: Effect of Recalibration Sample Size on Calibration Metrics on the PCam dataset.}
    \label{fig:res_recab_all_pcam}
\end{figure}

\begin{figure}[t]
    \centering
    \includegraphics[width=\textwidth]{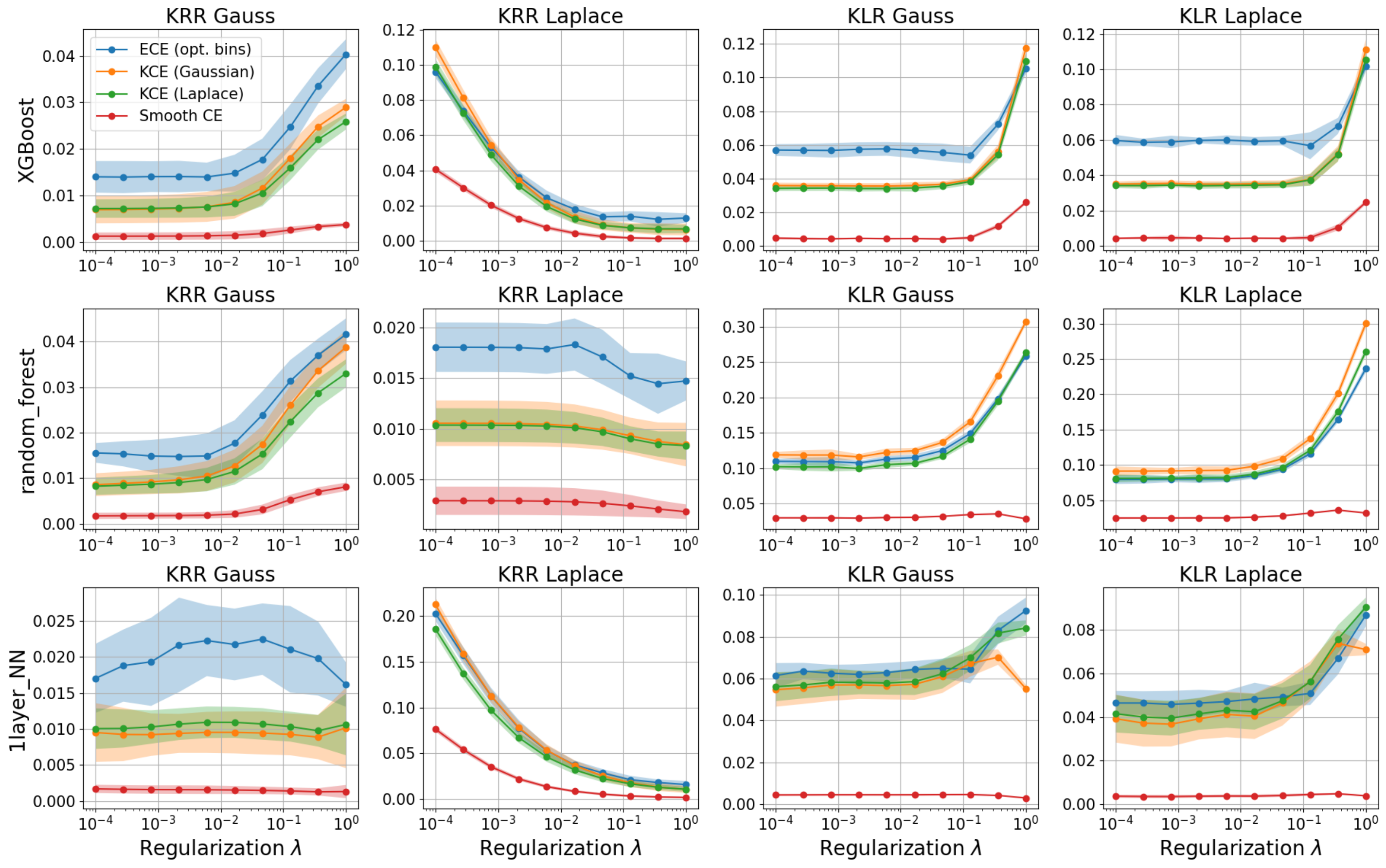}
    \caption{All Experimental Results of Recalibration: Effect of Regularization parameter $\lambda$ on Calibration Metrics on the PCam dataset.}
    \label{fig:res_recab_all_pcam_lam}
\end{figure}

\begin{figure}[t]
    \centering
    \includegraphics[width=\textwidth]{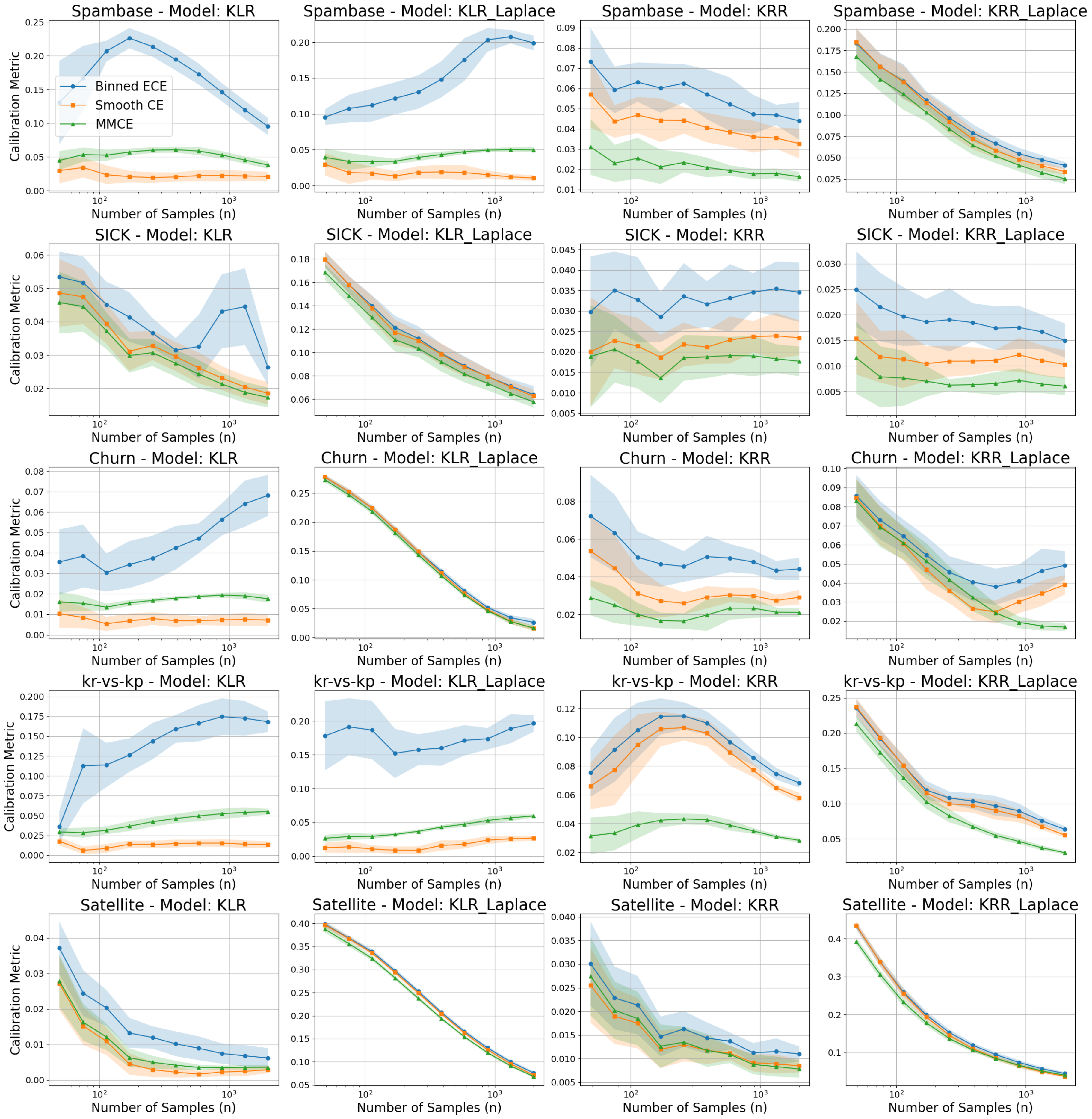}
    \caption{Effect of Sample Size on Calibration Metrics on the real world datasets.}
    \label{fig:res_real_n}
\end{figure}

\begin{figure}[t]
    \centering
    \includegraphics[width=\textwidth]{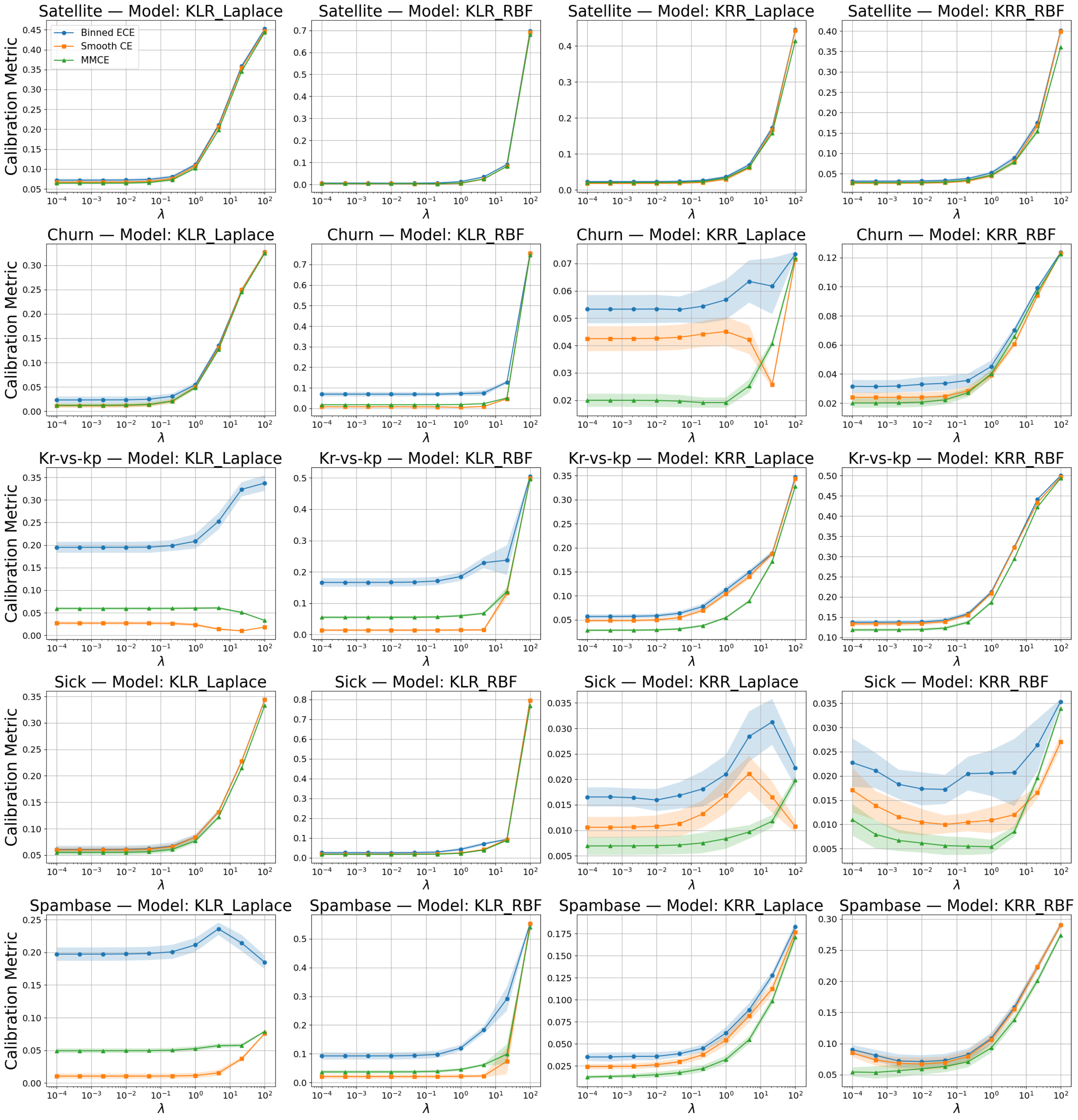}
    \caption{Effect of Sample Size on Calibration Metrics on the real world datasets.}
    \label{fig:res_real_lambda}
\end{figure}

\end{document}